\RequirePackage{iftex}
\ifPDFTeX\pdfoutput=1\fi
\documentclass[11pt]{article}
\usepackage[margin=1in]{geometry}
\usepackage[authoryear,round]{natbib}
\usepackage{amsmath,amssymb,amsthm,mathtools,bbm,microtype}
\usepackage{thm-restate}
\usepackage[colorlinks=true,linkcolor=blue,citecolor=blue,urlcolor=blue]{hyperref}
\usepackage[nameinlink,noabbrev,capitalize]{cleveref}
\newtheorem{theorem}{Theorem}[section]
\newtheorem{proposition}[theorem]{Proposition}
\newtheorem{lemma}[theorem]{Lemma}
\newtheorem{corollary}[theorem]{Corollary}

\theoremstyle{remark}

\crefname{equation}{Eq.}{Eqs.}
\Crefname{equation}{Eq.}{Eqs.}
\crefname{theorem}{Theorem}{Theorems}
\Crefname{theorem}{Theorem}{Theorems}
\crefname{lemma}{Lemma}{Lemmas}
\Crefname{lemma}{Lemma}{Lemmas}
\crefname{proposition}{Proposition}{Propositions}
\Crefname{proposition}{Proposition}{Propositions}
\crefname{corollary}{Corollary}{Corollaries}
\Crefname{corollary}{Corollary}{Corollaries}
\crefname{claim}{Claim}{Claims}
\Crefname{claim}{Claim}{Claims}
\crefname{remark}{Remark}{Remarks}
\Crefname{remark}{Remark}{Remarks}

\let\P\undefined
\DeclareMathOperator{\En}{\mathbb{E}}
\DeclareMathOperator{\P}{\mathbb{P}}
\DeclareMathOperator{\var}{Var}
\DeclareMathOperator{\cov}{Cov}

\DeclareMathOperator*{\argmin}{arg\,min}

\newcommand{\bbR}{\ensuremath{\mathbb{R}}}

\newcommand{\cF}{\ensuremath{\mathcal{F}}}
\newcommand{\cG}{\ensuremath{\mathcal{G}}}
\newcommand{\cH}{\ensuremath{\mathcal{H}}}

\newcommand{\cN}{\ensuremath{\mathcal{N}}}
\newcommand{\cX}{\ensuremath{\mathcal{X}}}

\newcommand{\ind}{\mathbbm{1}}
\newcommand{\eps}{\varepsilon}

\DeclareMathOperator{\conv}{conv}

\DeclareMathOperator{\unif}{\mathsf{unif}}

\newcommand{\reg}{\mathrm{Reg}}
\newcommand{\ewreg}{\mathrm{Reg}^{\mathrm{EW}}}
\newcommand{\fat}{\mathrm{fat}}

\newcommand{\dd}{\mathrm d}
\allowdisplaybreaks
\title{Oracle-Efficient and Parameter-Free Agnostic Smoothed Online Learning}
\author{
  Sasha Voitovych\\
  MIT\\
  {\small{\texttt{voitovyc@mit.edu}}}
  \and
  Adam Block\\
  Columbia University\\
  {\small{\texttt{adam.block@columbia.edu}}}
  \and
  Alexander Rakhlin\\
  MIT\\
  {\small{\texttt{rakhlin@mit.edu}}}
  \and
  Abhishek Shetty\\
  Georgia Tech\\
  {\small{\texttt{shetty@mit.edu}}}
}
\date{}

\begin{document}
\maketitle

\begin{abstract}
Online learning is an attractive framework in many domains because it permits well-defined learning even when data are dependent or chosen adversarially.
This generality, however, comes at a steep price, introducing significant statistical and computational barriers.
Recently, smoothed online learning has emerged as a promising framework that interpolates between the fully adversarial and fully stochastic settings by assuming that the conditional law of each covariate has density at most $1/\sigma$ with respect to some fixed base measure $\mu$, and it is known to match the statistical and computational guarantees of classical learning while still allowing for much of the flexibility of online learning.
However, existing oracle-efficient algorithms require either (i) sampling access to the base measure $\mu$ or (ii) labels that are perfectly predicted by a fixed hypothesis.
Both assumptions limit the applicability of these algorithms, in contrast to statistical learning, where empirical risk minimization (ERM) learns efficiently in the agnostic setting without any knowledge of the data distribution.
We show that neither assumption is necessary, giving the first oracle-efficient algorithm that achieves sublinear regret in the agnostic setting without knowledge of $\mu$.
Our algorithm, based on Gaussian Follow-The-Perturbed-Leader, is parameter-free: it requires no knowledge of $\mu$, the smoothing parameter $\sigma$, or the horizon $T$, and it achieves regret $\widetilde O(d\sqrt{T/\sigma})$ for binary classes of VC dimension $d$ with a single call to an ERM oracle per round, which is optimal up to a $\sqrt{d}$ factor.  En route to establishing the regret bound, we introduce several new techniques that may be of independent interest.
\end{abstract}

\section{Introduction}

Learning from independent data without a priori knowledge of the generating process is the archetypal paradigm in machine learning~\citep{valiant1984theory}: a learner that observes a training sample drawn from an \emph{unknown} distribution can generalize to new examples from the same distribution.
Its main appeal is that it removes the need to model nature's distribution, which may be highly complex and need not belong to a known parametric family.
Beyond the statistical benefits of assuming independence, there are computational advantages: near-optimal learning can be achieved with a simple algorithm, empirical risk minimization (ERM), which simply optimizes its predictions on historical data~\citep{vapnik1971uniform,blumer1989learnability}.

In many modern applications, however, the independence assumption is difficult to justify, let alone verify.
Data often arrive sequentially, and the data-generating process may even react to the learner's own predictions.
Many modern data modalities, such as text, video, and state trajectories of physical systems (e.g., in robotics and control), exhibit strong temporal dependencies.
In interactive decision making, such as bandits and reinforcement learning, the learner's own past decisions determine which data it collects~\citep{lattimore2020bandit,foster2021statistical}.
Financial markets are another example: prices are strongly dependent over time, and the predictions of trading systems influence the actions of market participants, which in turn shape the data those systems observe.
Deployed recommendation systems exhibit similar feedback loops: the system's recommendations influence users' behavior, which in turn shapes the data the system collects~\citep{chaney2018algorithmic,jiang2019degenerate}.
In such settings, guarantees premised on independence need not hold, and algorithms developed assuming independence may not learn, necessitating a theory of learning that makes minimal assumptions on how the data are generated.

Online learning~\citep{littlestone} provides a promise of such a theory: it dispenses with the independence assumption altogether and permits learning even from dependent, possibly adversarial data.
Formally, online learning is a repeated game between a learner and nature: in each round $t$, the learner observes a covariate $X_t$, predicts a label $\hat Y_t$, observes the true label $Y_t$, and suffers loss $\ell(\hat Y_t, Y_t)$.
The learner's goal is to minimize its \emph{regret} with respect to a fixed comparator class $\cF$:
\begin{equation}\label{eq:reg}
\reg_T = \sum_{t=1}^{T} \ell(\hat Y_t, Y_t) - \inf_{f \in \cF} \sum_{t=1}^T \ell(f(X_t), Y_t).
\end{equation}
Importantly, the data-generating process may be arbitrary; in particular, it may adapt to the learner's past actions.
This generality, however, comes at both a statistical and a computational price.

Learnability in online binary classification is characterized by the Littlestone dimension, which may be infinite even for simple classes such as linear threshold functions~\citep{littlestone,bendavid2009agnostic}.
The situation is further exacerbated by computational barriers.
A natural computational framework is the \emph{ERM oracle model}, designed to capture the efficacy of optimization in modern practice.
When queried, an ERM oracle for $\cF$ returns a predictor $f\in\cF$ that minimizes the empirical loss on the given dataset.
This oracle forms the basic computational primitive of statistical learning and is a practically meaningful abstraction, since ERM is routinely solved, at least approximately, with off-the-shelf tools such as stochastic gradient descent.
Ideally, then, one would like online learning algorithms to be \emph{oracle-efficient}, i.e., run efficiently when given query access to an ERM oracle.
Without further assumptions in the sequential setting, however, \citet{hazan2016computational} showed that this goal is unattainable: even for a finite comparator class, access to an ERM oracle does not make online learning computationally tractable.
In other words, in the absence of independence, generalization and learning do not reduce to optimization.

In order to circumvent these barriers, there has been substantial effort in understanding smoothed online learning, which provides a way to interpolate between the stochastic and adversarial regimes~\citep{rakhlin2011online,haghtalab2020smoothed,block2022smoothed, block2024performance, blanchard2025agnostic, blanchard2026characterizing}.
Intuitively, smoothness prevents the adversary from concentrating probability mass on small sets, while still permitting strong temporal dependencies.
Formally, letting $p_t$ denote the law of $X_t$ conditional on the history, we require that there exist a single probability measure $\mu$, referred to as the \emph{base measure}, such that in every round $t$,
\[
\left\|\frac{\dd p_t}{ \dd\mu}\right\|_\infty \leq \frac1\sigma,
\]
where the smoothness parameter $\sigma\in(0,1]$ controls the power of the adversary.
At $\sigma=1$, the covariates are drawn independently from $\mu$ in every round; as $\sigma\downarrow0$, the constraint becomes increasingly permissive.
This relaxation removes both statistical and computational barriers: an extensive literature shows that VC classes are learnable in the smoothed setting for every $\sigma>0$ and, moreover, efficiently so in the ERM-oracle model~\citep{haghtalab2020smoothed,block2022smoothed,haghtalab2022oracle}.
Further, from a statistical perspective, \cite{blanchard2026characterizing} showed that a technical generalization of the smoothed model essentially captures learnability for general sequences of distributions.

However, most oracle-efficient algorithms in this setting assume that $\mu$ \emph{is known to the learner} and rely crucially on the ability to sample from it.
In practice, this amounts to an additional modeling assumption on nature, and the choice of any particular $\mu$ is often hard to justify.
A notable exception is the work of~\citet{block2024performance}, which shows that ERM itself achieves sublinear regret in the \emph{realizable} setting, where the data are labeled by some fixed $f^\star \in \cF$, i.e., $Y_t = f^\star(X_t)$.\footnote{More generally, their result applies to well-specified regression.}
Whether efficient learning is possible in the more general agnostic setting has remained open.
This is in contradistinction to classical statistical learning theory, where a single algorithm (ERM) succeeds in both the realizable and agnostic settings while being entirely parameter-free and requiring no knowledge of the data distribution.
Thus, in an attempt to relax the independence assumptions on nature, we have lost precisely what makes statistical learning appealing: the learner never needs to model nature explicitly.

\paragraph{Main results.}
In this work, we give the first oracle-efficient algorithm that achieves sublinear regret in general agnostic smoothed online learning without the knowledge of the base measure.
Surprisingly, this guarantee is achieved by a natural algorithm: follow-the-perturbed-leader (FTPL) with Gaussian perturbations, which is entirely parameter-free.
It is well known since \cite{hannan1957approximation} that ERM can fail on agnostic data even when the covariate domain has size $1$ (which forces $\sigma = 1$).
Intuitively, this failure stems from a lack of \emph{stability}: ERM is sensitive to individual data points, and its solution can change drastically in response to a single label flip.
A standard way to stabilize ERM is to regularize it, either with an explicit potential, as in follow-the-regularized-leader (FTRL)~\citep{shalev2007primal, abernethy2008competing, mcmahan2017survey}, or with a random perturbation of the objective, as in FTPL~\citep{hannan1957approximation, kalai2005efficient,hutter2005adaptive, devroye2013prediction}.
Motivated by the algorithm in \citet{block2022smoothed}, we show that natural Gaussian perturbations suffice to achieve sublinear regret.
In particular, for binary classification, we show that FTPL with Gaussian perturbations attains
\begin{equation}\label{eq:main-regret}
\En\reg_T\le \widetilde O\left(d \sqrt{\frac T\sigma}\right)
\end{equation}
simultaneously for all $\mu$ and $\sigma$.
This guarantee has optimal dependence on $T$ and $\sigma$, and is statistically suboptimal only by a factor of $\sqrt{d}$~\citep{blanchard2025agnostic}. We also note that in the realizable case the $\sqrt{d}$ factor can be removed via an $L^\star$-regret bound (see~\cref{app:lstar}). This matches the ERM guarantee of~\citet{block2024performance} in the realizable case and is optimal up to logarithmic factors~\citep[Theorem~3]{blanchard2025agnostic}.

We now define our Gaussian FTPL learner formally.
Let $\cF \subseteq \{\pm1\}^{\cX}$ be a comparator class.
In round $t$, the algorithm samples a fresh Gaussian vector $G^{(t)}=(G_1^{(t)},\ldots,G_{t-1}^{(t)}) \sim \cN(0,I_{t-1})$ and computes
\begin{equation}\label{eq:ftpl}
\hat f_t \in \argmin_{f \in \cF} \sum_{s=1}^{t-1} \left[\ell(f(X_s),Y_s)+G_s^{(t)}f(X_s)\right],
\end{equation}
where we let $\ell(\hat y, y) = -y\hat y$ in the binary case.
We use a weighted ERM oracle (\cref{sec:prelims} gives the precise model).
This algorithm is completely parameter-free: it requires no knowledge of the base measure $\mu$, the smoothness parameter $\sigma$, or the horizon $T$.
It thus addresses precisely the limitation of prior work discussed above, achieving oracle-efficient learning adaptively for every $\mu$ and $\sigma$.

Our algorithm and guarantee extend to classes $\cF\subseteq[-1,1]^{\cX}$ and convex $1$-Lipschitz losses by adding an averaging step in each round (\cref{sec:results}).
In terms of the fat-shattering dimension $\fat_\eps(\cF)$ (\cref{sec:prelims}), the resulting bound is $\widetilde O(d\sqrt{T/\sigma})$ for parametric classes with $\fat_\eps(\cF)\lesssim d\log(1/\eps)$, and $\widetilde O_\alpha(T^{(2\alpha+1)/(2\alpha+2)}/\sqrt\sigma)$ for nonparametric classes with $\fat_\eps(\cF)\lesssim\eps^{-\alpha}$, matching the leading exponents of $T$ and $1/\sigma$ in~\citet[Theorem~4]{blanchard2025agnostic}.

\paragraph{Technical overview.}
We sketch the analysis for the binary case, which already contains the main ideas for the general real-valued setting.
A central tool in the analysis of FTPL algorithms is stability~\citep{kalai2005efficient}: intuitively, if the perturbations have controlled magnitude and the algorithm's predictions are stable, a regret bound follows.
However, we find it difficult to analyze the stability of~\cref{eq:ftpl} directly; instead, we are able to show that the predictions of~\cref{eq:ftpl} are close to predictions of an auxiliary algorithm that is itself stable.

In more detail, our proof proceeds in two steps: (i) we first compare the predictions of our FTPL forecaster to those of an auxiliary algorithm, exponential weights over a cover $\cG$ of $\cF$~\citep{vovk1990aggregating, littlestone1994weighted, freund1997decision}; (ii) we then directly bound the regret of the auxiliary algorithm.
We emphasize that the auxiliary algorithm need not be efficient, as it is used only in the analysis.

\textsc{Step 1: Comparison with the auxiliary algorithm.}
Formally, let $\tilde f_t \in \conv(\cG)$ be defined as follows:
\begin{equation}\label{eq:ew}
\tilde f_t:=\sum_{f\in\cG}w_t(f\mid G^{(t)})f, \quad \text{with} \quad w_t(f\mid G^{(t)})\propto \exp\left(-\sum_{s<t}[\ell(f(X_s),Y_s)+G_s^{(t)}f(X_s)]\right),
\end{equation}
such that the weights $w_t(f\mid G^{(t)})$ sum to one over all $f \in \cG$.
Our first technical contribution reduces the regret of $\hat f_t$ to that of $\tilde f_t$.
More precisely, letting $\ewreg$ denote the regret of the learner in~\cref{eq:ew}, we show that, as long as $\cG$ approximates $\cF$ sufficiently well (see~\cref{sec:reduction} for details),
\begin{equation}\label{eq:ew-reduction}
\En \reg_T - \En \ewreg_T \le \widetilde O \left(d \sqrt{\frac{T}{\sigma}}\right).
\end{equation}
The proof of this reduction has two steps.
First, we compare the predictions of $\hat f_t$ and $\tilde f_t$ on past covariates $X_{1:t-1}$, showing that for every $1\le t\le T$,
\begin{equation}\label{eq:historical-comparison}
\En\sum_{s = 1}^{t-1} (\En_{G^{(t)}} \hat f_t(X_s) - \En_{G^{(t)}} \tilde f_t(X_s))^2 \lesssim d^2 \log^2(2T/\sigma).
\end{equation}
This follows by expressing the prediction vectors of $\hat f_t$ and $\tilde f_t$ as gradients of certain potentials and using Stein's lemma~\citep{stein1981estimation, abernethy2014smoothing} to relate differences in predictions to differences in potentials.

Next, we use an improved version of the \emph{surprise lemma} of~\citet{block2024performance} to translate~\cref{eq:historical-comparison} into closeness of predictions on the current covariate $X_t$.
Specifically, adapting arguments from~\cite{xie2022coverage,block2024performance} to convex hulls of VC classes, we prove that for any sequence of predictable functions $h_t \in \conv(\pm\cF)$,
\begin{align*}
\En\sum_{t=1}^T|h_t(X_t)| &\lesssim \sqrt{\frac{\log(2T)}\sigma \En\sum_{t=1}^T\sum_{s<t}h_t(X_s)^2} +\sqrt{\frac{dT}\sigma}\log^2(2T/\sigma).
\end{align*}
Intuitively, this shows that the value of a predictable test function $h_t$ at a new covariate $X_t$ is controlled by its cumulative second moments on past covariates; informally, values of $|h_t(X_t)|$ that are surprisingly large relative to the history cannot occur too often.
Applied to differences $h_t=(\En_{G^{(t)}}\hat f_t-\En_{G^{(t)}}\tilde f_t)/2$ and combined with~\cref{eq:historical-comparison}, this shows that predictions $\hat f_t(X_t)$ and $\tilde f_t(X_t)$ are close on average.
This proves~\cref{eq:ew-reduction} since the loss is $1$-Lipschitz.

\textsc{Step 2: Bounding the regret of the auxiliary algorithm.}
Thus, we have reduced our problem to bounding the regret of the exponential weights algorithm in~\cref{eq:ew}.
Here, a central difficulty is that, while typical worst-case bounds on the regret of exponential weights require learning rate $\eta \propto 1/\sqrt{T}$, the algorithm described by \cref{eq:ew} uses a \emph{constant learning rate} (i.e. $\eta =1$), for which these bounds are vacuous \citep{cesa2006prediction}.
We show that despite this apparent hurdle, the Gaussian perturbations themselves provide the necessary stability to ensure low regret.
Precisely, we show that the predictions of exponential weights with Gaussian perturbations have low cumulative variance:
\[
\sum_{t=1}^T \En\var_{f \sim w_t(\cdot \mid G^{(t)})}(f(X_t)) \le \widetilde O\left( \sqrt{\frac{T(\log|\cG| + d)}{\sigma}}\right).
\]
Combining this with the standard variance-based analysis of exponential weights~\citep{cesa2007secondorder,deRooij2014follow} yields a corresponding regret bound:
\[
\En \ewreg_T \lesssim d\log(2T/\sigma) + \sum_{t=1}^T \En\var_{f \sim w_t(\cdot \mid G^{(t)})}(f(X_t)) \le \widetilde O\left( \sqrt{\frac{T(\log|\cG| + d)}{\sigma}}\right).
\]
By controlling the size of the cover (\cref{prop:cover}) and using~\cref{eq:ew-reduction}, we obtain~\cref{eq:main-regret}.

\paragraph{Conventions and notation.}
We write $\log$ for the natural logarithm.
For an integer $n$, we let $[n]$ denote the set $\{1, \ldots , n\}$. Throughout, \(C,c>0\) denote sufficiently large and sufficiently small universal constants, respectively.
We additionally use standard non-asymptotic big-$O$ notation as shorthand: we write $f \le O(g)$ or $f\lesssim g$ to mean that $f \le Cg$ for a sufficiently large universal constant $C$.
We use $\widetilde O$ to suppress logarithmic factors in $d$, $T$, and $1/\sigma$.
We write $Z_{a:b}$ for the sequence $(Z_a,Z_{a+1},\ldots,Z_b)$.
For a class $\cH$ of functions, we write $\conv(\cH)$ to denote the pointwise closure of the convex hull of $\cH$ and let $\conv(\pm\cH):=\conv(\cH\cup-\cH)$.

\section{Preliminaries and Related Work}\label{sec:prelims}

We now introduce the necessary technical definitions and notation for our analysis, before moving on to briefly discuss related work.
\paragraph{Online learning and smoothness.}
We reserve $\cX$ to refer to the space of covariates, and $d$ to refer to the VC dimension of the comparator class $\cF$.
For convenience, we assume binary labels are in $\{\pm 1\}$.
For a sequence of covariates $X_t$, labels $Y_t$, and predictions $\hat Y_t$, let $\cH_{t} = \sigma((X_s,Y_s,\hat Y_s)_{s \in [t]})$ be the $\sigma$-algebra induced by the learner's observations up to round $t$, and write $p_t=\mathrm{Law}(X_t\mid\cH_{t-1})$.
For a fixed probability measure $\mu$ and $\sigma\in(0,1]$, the covariate process is $\sigma$-smooth with respect to $\mu$ if, almost surely, for all $t$, $p_t$ is absolutely continuous w.r.t. $\mu$ and
\[
\left\|\frac{\mathrm dp_t}{\mathrm d\mu}\right\|_\infty\le\frac1\sigma.
\]

\paragraph{Complexity measures.}
Let $\cF\subseteq\{\pm1\}^{\cX}$.
We say that $\cF$ \emph{shatters} a set $\{x_1,\ldots,x_n\}\subseteq\cX$ if for every $y_{1:n}\in\{\pm1\}^n$ there exists $f\in\cF$ with $f(x_i)=y_i$ for all $i\in[n]$.
The \emph{VC dimension} of $\cF$ is the largest integer $d$ such that $\cF$ shatters some set of size $d$ (or $\infty$ if no such largest $d$ exists)~\citep{vapnik1971uniform}.

For real-valued classes $\cF\subseteq[-1,1]^{\cX}$ and a scale $\eps>0$, we say that $\cF$ \emph{$\eps$-shatters} a set $\{x_1,\ldots,x_n\}\subseteq\cX$ if there exist witnesses $y_{1:n}\in\bbR^n$ such that for every $\sigma_{1:n}\in\{\pm1\}^n$ there exists $f\in\cF$ with
\[
\sigma_i\left(f(x_i)-y_i\right)\ge\eps\qquad\text{for all } i\in[n].
\]
The \emph{fat-shattering dimension} of $\cF$ at scale $\eps$, denoted $\fat_\eps(\cF)$, is the largest $n$ such that $\cF$ $\eps$-shatters some set of size $n$~\citep{kearns1994efficient,alon1997scale}.
The map $\eps\mapsto\fat_\eps(\cF)$ is non-increasing, and for binary classes $\cF\subseteq\{\pm1\}^{\cX}$, $\fat_\eps(\cF)$ is equal to the VC dimension of the class for all $\eps\in(0,1]$.

\paragraph{ERM oracle model.}
Following~\citet{hazan2016computational} and the literature on oracle-efficient smoothed online learning~\citep{block2022smoothed,haghtalab2022oracle}, we assume access to a weighted agnostic ERM oracle.
In particular, for any $n$ covariates $X_{1:n}$, a sequence $\ell_i \colon [-1,1] \to [-1,1]$ of convex functions, each given either by label loss (that is, $\ell_i(\cdot) = \ell(\cdot, y_i)$ for some label $y_i$) or a linear function, and weights $w_i \ge0$, we assume we can compute a solution to the following optimization problem:
\[
\argmin_{f \in \cF} \sum_{i=1}^n w_i\ell_i(f(X_i)).
\]
We note that weights in our analysis satisfy $\|w\|_\infty \lesssim \sqrt{\log(T)}$ w.h.p.

\paragraph{Related work.}
Due to its robustness and fundamental nature, online learning has been extensively studied over the past decades with a relatively complete understanding of its theoretical foundations \citep{freund1997decision,bendavid2009agnostic,cesa2006prediction,shalev2012online,rakhlin2011online,block21a} and computational properties \citep{hazan2016computational,kalai2005efficient,abernethy2014smoothing,rakhlin2012relax}.
On the other hand, because of the strong statistical and computational lower bounds, more recently a number of works have explored online learning under various relaxations and assumptions to achieve tractable algorithms and improved performance guarantees \citep{rakhlin2011online}.
Of particular relevance are prior works~\citep{haghtalab2020smoothed,block2022smoothed,haghtalab2022oracle,block2023sample,block2024oracle,block2024performance} that focus on the statistical and computational advantages of smoothed online learning under the assumption that the base measure is \emph{known}, as well as applications to other domains like game theory \citep{daskalakis2023smooth}.
Other work sought to extend these results, either by taking advantage of additional structure for specific base measures \citep{block2023smoothed,block2023oracle,block2022efficient}, or considering more general loss functions \citep{bhatt2023smoothed}.

In all of the above work relevant to smoothed online learning, some knowledge of the base measure was required, which limits the applicability of the results.
A separate line of work has sought to develop methods that do not rely on knowing the base measure, thereby broadening the scope and applicability of smoothed online learning techniques \citep{wu2023online,wu2024oracle}, although the former work focuses on oracle-inefficient algorithms under substantially stronger assumptions whereas the latter provides oracle-efficient solutions assuming the features are independent.
In \citet{block2024performance}, the authors demonstrate that it is possible to achieve strong performance guarantees even without knowledge of the base measure using just empirical risk minimization, but they require the labels to be \emph{well-specified} in the sense that their conditional means can be perfectly predicted by an element of the comparator class.
On the other end of the spectrum, \citet{blanchard2025agnostic} consider the agnostic setting where no such well-specification assumption is made, but do not consider oracle-efficient algorithms.
Unlike these prior works, ours is the first to propose an oracle-efficient algorithm (similar in spirit to one in \citet{block2022smoothed}) that is capable of achieving sublinear regret under an arbitrary unknown base measure; moreover, our algorithm's regret scales optimally in both time and smoothness, losing only a $\sqrt{d}$ factor compared to the optimal rate.

\section{Main Results}\label{sec:results}

The main results of our paper are two oracle-efficient regret bounds, for the binary and the real-valued loss cases.
We begin with the binary case.
Recall that Gaussian FTPL as defined in~\cref{eq:ftpl}, in each round $t$, computes:
\[
\hat f_t \in \argmin_{f \in \cF} \sum_{s=1}^{t-1} \left[\ell(f(X_s),Y_s) + G_s^{(t)}f(X_s)\right],
\]
and predicts $\hat Y_t = \hat f_t (X_t)$.
Note that the above can be implemented using one call to the weighted ERM oracle.
For this algorithm, we obtain the following regret guarantee.

\begin{theorem}\label[theorem]{thm:main}
Consider a binary class $\cF \colon \cX \to \{\pm1\}$ of VC dimension $d\ge1$, and let $\ell(\hat y,y)=-\hat yy$ be the binary loss.
For every horizon $T\ge1$ and every covariate process that is $\sigma$-smooth with respect to a fixed $\mu$, the learner in~\eqref{eq:ftpl} satisfies
\begin{equation*}
\En\reg_T \lesssim d\log^2(2T/\sigma)\sqrt{\frac T\sigma}.
\end{equation*}
\end{theorem}
A complete proof of the above statement is given in~\cref{app:binary-proofs}.
Note that the regret bound above is obtained in a completely parameter-free manner.
In particular, the learner in~\eqref{eq:ftpl} requires no knowledge of the base measure $\mu$, the smoothness parameter $\sigma$, or the horizon $T$.

Extending the above guarantee to convex Lipschitz losses requires an additional averaging step, which, intuitively, helps reduce the variance of the predictor.
Formally, in round $t$, our learner resamples the noise $G^{(t)}$ independently $t$ times, and computes (conditionally) independent solutions $\hat f^{(i)}_t$, $i\in[t]$ to the optimization problem in~\cref{eq:ftpl}.
The prediction is then aggregated in the following way:
\begin{equation}\label{eq:algo-rv}
\hat Y_t = \frac{1}{t} \sum_{i=1}^{t} \hat f_t^{(i)}(X_t).
\end{equation}
The above can be recognized as a Monte-Carlo estimate of the expectation $\En_{G^{(t)}}\hat f_t(X_t)$.
By Jensen's inequality $\ell(\En_{G^{(t)}}\hat f_t(X_t), Y_t) \le \En_{G^{(t)}} \ell (\hat f_t(X_t), Y_t)$, so, up to the Monte Carlo error, the averaged prediction incurs no more loss than a single Gaussian FTPL prediction in conditional expectation.
\begin{restatable}{theorem}{mainrealvalued}\label[theorem]{thm:main-real-valued}
Consider a class $\cF \colon \cX \to [-1,1]$ and let $\ell(\hat y,y)$ be convex and $1$-Lipschitz in its first argument, and suppose the covariate process is $\sigma$-smooth with respect to a fixed $\mu$.
Then, the learner in~\eqref{eq:algo-rv} satisfies
\begin{equation}\label{eq:main-rate-real-valued}
\En\reg_T\lesssim \inf_{0<\eps\le1/16} \left\{\left(1+\fat_{c\eps}(\cF)+\sqrt T\eps\right) \sqrt{\frac T\sigma}\log^{3}\left(\frac{2T}{\sigma\eps}\right)\right\},
\end{equation}
where $c$ is a universal constant.
\end{restatable}
A complete proof of the above statement is given in~\cref{app:real-proofs}.
Theorem~\ref{thm:main-real-valued} gives explicit rates for two important families of function classes.
For \emph{parametric classes}, i.e., those satisfying $\fat_{\eps}(\cF)\lesssim d\log(1/\eps)$ at sufficiently small scales, we have
\[
\En\reg_T\le\widetilde O\left(d\sqrt{\frac T\sigma}\right).
\]
As in binary classification, the dependence on $d$ is suboptimal by a factor of $\sqrt d$, up to logarithmic factors~\citep{blanchard2025agnostic}.
For \emph{nonparametric classes}, i.e., those satisfying $\fat_{\eps}(\cF)\lesssim\eps^{-\alpha}$, $\alpha>0$, choosing $\eps\asymp_\alpha T^{-1/(2\alpha+2)}$ gives
\[
\En\reg_T\le\widetilde O_\alpha\left( \frac{T^{(2\alpha+1)/(2\alpha+2)}}{\sqrt\sigma}\right).
\]
For fixed $\sigma$, this is sublinear for every $\alpha>0$.
It matches the leading horizon and smoothness exponents in~\citet[Theorem~4]{blanchard2025agnostic}, which achieves this guarantee with an inefficient algorithm.
Notably, our result applies to regret as defined in~\cref{eq:reg}, while~\cite{blanchard2025agnostic} only controls pseudo-regret (where comparator is fixed before the draw of the data).
Compared with the i.i.d. setting, where the optimal regret rate scales as $T^{\frac{1}{2}\lor\frac{\alpha-1}{\alpha}}$~\citep{mendelson2003entropy}, up to logarithmic factors, our bound has a larger exponent of $T$.
Nevertheless, sublinear regret remains possible for every $\alpha>0$.
We thus present a framework for smoothed online learning that recovers the qualitative picture of statistical learnability, while retaining oracle efficiency and requiring no knowledge of the base measure.

\section{Analysis Techniques}
\label{sec:analysis}

We give an overview of the proof of~\cref{thm:main} in the binary case, as it is simpler and still contains the key ideas needed for the general case.
The complete proofs for the binary case are deferred to~\cref{app:binary-proofs}, and the proofs for the real-valued case are contained in Appendix~\ref{app:real-proofs}.
Our proof has three ingredients: a sharper surprise lemma (\cref{sec:surprise}), a reduction to exponential weights (\cref{sec:reduction}), and a regret bound for Gaussian-perturbed exponential weights (\cref{sec:ew}).

\subsection{Surprise Lemma}\label{sec:surprise}

For smoothed online learning algorithms that aim to extrapolate behavior over the past into the future, surprise lemmas have proven useful~\citep{block2024performance,blanchard2025agnostic}.
Intuitively, for an arbitrary sequence of (predictable) test functions $\{h_t\}$, a surprise lemma quantifies for how many $t$ the value $h_t(X_t)$ can be significantly larger than the past values $h_t(X_1),\ldots,h_t(X_{t-1})$.
Formally, we prove the following version of the surprise lemma, which controls expected absolute values of test functions as a function of their past second moments.
\begin{restatable}{lemma}{surpriselemma}
\label{lemma:surprise}
Let $\cF$ be a class of VC dimension $d$ and let $\{h_t\}_{t\in [T]}$ be a predictable sequence of functions with $h_t\in\conv(\pm\cF)$.
Then
\begin{equation}\label{eq:surprise}
\En\sum_{t=1}^T|h_t(X_t)| \lesssim \sqrt{\frac{\log(2T)}\sigma \En\sum_{t=1}^T\sum_{s<t}h_t(X_s)^2}+\sqrt{\frac{dT}\sigma}\log^2\left(\frac{2T}{\sigma}\right).
\end{equation}
\end{restatable}
The above is proven in two steps.
We first use a result from~\citet[Section~3.2, equations~(3)--(5)]{xie2022coverage}, which establishes an analogous estimate where the second moments are evaluated on a \emph{tangent} sequence $\bar X_{1:T}$ as opposed to realized covariates $X_{1:T}$, where a \emph{(decoupled) tangent sequence} $\bar X_{1:T}$ is defined by drawing independent samples $\bar X_{t} \mid \cH_{t-1} \sim p_t$ for each $t$.
Compared to the analogous statement in~\citet[Lemma~12]{block2024performance}, which involves first moments of the test functions, keeping second moments allows for a sharper transfer from past to current covariates.
The next step is to extend the norm comparison inequality of~\cite[Theorem~10]{block2024performance} to convex hulls of VC classes, for which we utilize Maurey's empirical method~\citep{pisier1981remarques}.
It is important to allow $h_{t} \in \conv(\pm \cF)$ as we wish to apply~\cref{lemma:surprise} to functions of predictions of exponential weights, which live in the convex hull of the class.
The full proof is in Appendix~\ref{app:surprise}.

\subsection{Reduction to exponential weights}\label{sec:reduction}

In this section, we show that predictions of our Gaussian FTPL algorithm are close to predictions of an auxiliary exponential weights algorithm that is run over the cover $\cG$ of $\cF$.
Prior work establishes existence of covers of controlled size that approximate $\cF$ well along a draw of a smoothed covariate trajectory, using the coupling lemma~\citep{haghtalab2024smoothed}.
\begin{restatable}{proposition}{coverguarantee}
\label{prop:cover}
Let $\cF \colon \cX \to \{\pm 1\}$ be a binary class with VC dimension $d$.
For every probability measure $\mu$, horizon $T\ge1$, and $\sigma\in(0,1]$, there exists a cover $\cG \subseteq \cF$ such that $\log |\cG| \lesssim d \log (2T/\sigma)$ and, for every covariate process that is $\sigma$-smooth with respect to $\mu$,
\[
\P\left(\sup_{f\in\cF}\inf_{g\in\cG}\sum_{s\le T} \ind\{f(X_s)\ne g(X_s)\} > C d \log(2T/\sigma) \right) \le \frac{1}{T},
\]
where $C>0$ is a universal constant.
\end{restatable}
Recall that in~\cref{eq:ew} we define an auxiliary exponential weights learner using the cover $\cG$ from~\cref{prop:cover}:
\begin{equation*}
\tilde f_t:=\sum_{f\in\cG}w_t(f\mid G^{(t)})f, \quad \text{with} \quad w_t(f\mid G^{(t)})\propto \exp\left(-\sum_{s<t}[\ell(f(X_s),Y_s)+G_s^{(t)}f(X_s)]\right).
\end{equation*}
We show that predictions of such an algorithm are close on average to the predictions of our Gaussian FTPL learner as in~\cref{eq:ftpl}.
\begin{restatable}[Reduction to exponential weights]{lemma}{ewreduction}\label{lemma:reduction}
Consider a binary function class $\cF\colon \cX\to \{\pm1\}$ of VC dimension $d$.
For $\hat f_t$ and $\tilde f_t$ as defined in~\cref{eq:ftpl,eq:ew}, respectively, we have the following bound:
\begin{equation}\label{eq:regret-reduction}
\En\sum_{t=1}^T |\En_{G^{(t)}} \hat f_t(X_t) - \En_{G^{(t)}} \tilde f_t(X_t)| \lesssim d\log^2(2T/\sigma)\sqrt{\frac T\sigma}.
\end{equation}
\end{restatable}
Since binary loss $\ell(\hat y, y) = -\hat y y$ is Lipschitz, the above immediately implies
\begin{equation*}
\En\reg_T-\En\ewreg_T \lesssim d\log^2(2T/\sigma)\sqrt{\frac T\sigma}.
\end{equation*}
We prove~\cref{eq:regret-reduction} in two steps.
First, we show that $\hat f_t$ and $\tilde f_t$ agree on \emph{past covariates} (when averaged over their Gaussian randomness).
Then, using predictability of $\hat f_t$ and $\tilde f_t$, we leverage the surprise lemma (\cref{lemma:surprise}) to transfer that guarantee into closeness on current covariates.
In particular, we obtain the following estimate.
\begin{restatable}{lemma}{empiricalprediction}\label{lemma:empirical-prediction}
For every $t\le T$,
\begin{equation*}
\En\sum_{s<t} \left(\En_{G^{(t)}} \hat f_t(X_s)-\En_{G^{(t)}}\tilde f_t(X_s)\right)^2 \lesssim d^2\log^2(2T/\sigma).
\end{equation*}
\end{restatable}
It shows that the Gaussian-averaged predictions $\En_{G^{(t)}}\hat f_t$ and $\En_{G^{(t)}}\tilde f_t$ closely agree on the past data points in expected squared norm.
\begin{proof}[Proof sketch of~\cref{lemma:empirical-prediction}]
The case $t=1$ is immediate.
For $t\ge2$, fix the history $(X_{1:t-1}, Y_{1:t-1})$, write $G=G^{(t)}$, and introduce two potentials:
\[
\Phi_t(G)=\max_{f\in\cF}\sum_{s<t}(Y_s-G_s)f(X_s),
\]
\[
\widetilde\Phi_t(G) =\log\sum_{f\in\cG}\exp\left(\sum_{s<t}(Y_s-G_s)f(X_s)\right) =\max_{w\in\Delta(\cG)} \left[\En_{f\sim w} \sum_{s<t}(Y_s-G_s)f(X_s) + H(w)\right],
\]
where $H(w)=\sum_{f\in\cG}w(f)\log(1/w(f))$ denotes entropy.
Then, almost surely,\footnote{In particular, it suffices that any two maximizers $f$ and $f'$ of either potential give the same predictions on $X_{1:t-1}$, which is true almost surely by absolute continuity of the Gaussians.} we have $\hat f_t(X_s)=-\partial_{G_s}\Phi_t(G)$ and $\tilde f_t(X_s)=-\partial_{G_s}\widetilde\Phi_t(G)$ for every $s<t$.
These identities follow from Danskin's theorem for the finite maximum \citep[Appendix~B]{bertsekas2016nonlinear}, and using the fact that $w_t$ is the maximizer in the variational formula for $\tilde\Phi_t$; in the context of online learning, this idea also appears in~\cite{abernethy2014smoothing}.
Thus,
\begin{align*}
\sum_{s<t}(\En_G\hat f_t(X_s)-\En_G\tilde f_t(X_s))^2 &=\left\|\En_G\nabla_G(\Phi_t(G)-\widetilde\Phi_t(G))\right\|_2^2 \\
&\overset{(a)}{=}\left\|\En_G[G(\Phi_t(G)-\widetilde\Phi_t(G))]\right\|_2^2 \\
&=\left\|\En_G\left[G\left(\Phi_t(G)-\widetilde\Phi_t(G) -\En_G[\Phi_t(G)-\widetilde\Phi_t(G)]\right)\right]\right\|_2^2 \\
&=\sup_{\|u\|_2=1}\cov_G(u^\top G,\Phi_t-\widetilde\Phi_t)^2 \\
&\overset{(b)}{\le}\sup_{\|u\|_2=1} \En_G[(u^\top G)^2]\, \var_G(\Phi_t-\widetilde\Phi_t) \\
&=\var_G(\Phi_t-\widetilde\Phi_t),
\end{align*}
where (a) uses Stein's lemma (\cref{lemma:stein}) and (b) uses the Cauchy--Schwarz inequality.
It remains to control $\Phi_t-\widetilde\Phi_t$.
The two potentials differ in that (1) $\Phi_t$ maximizes over $\cF$ rather than $\cG$, which~\cref{prop:cover} controls, and (2) $\widetilde\Phi_t$ adds an entropy term, which costs at most $\log|\cG|$.
Combining these, we obtain:
\[
\var_G(\Phi_t-\widetilde\Phi_t)\le\En_G|\widetilde\Phi_t-\Phi_t|^2 \lesssim d^2\log^2(2T/\sigma)+\log^2|\cG|\lesssim d^2\log^2(2T/\sigma),
\]
as desired.
\end{proof}

An application of the surprise lemma (\cref{lemma:surprise}) to~\cref{lemma:empirical-prediction} finishes the proof of~\cref{lemma:reduction}.
The full proof is in Appendix~\ref{app:reduction}.

\subsection{Regret of Gaussian-perturbed exponential weights}\label{sec:ew}

It remains to bound the regret of the exponential weights algorithm.
In particular, we show the following upper bound.

\begin{restatable}{lemma}{ewregret}\label{lemma:ew}
The regret of the exponential weights algorithm in~\eqref{eq:ew} satisfies
\begin{equation}\label{eq:ew-rate}
\En\ewreg_T \lesssim \sqrt{\frac{dT}\sigma}\log^2(2T/\sigma).
\end{equation}
\end{restatable}
The proof has three steps.
First, we adapt the variance-based analysis of exponential weights~\citep[Lemmas~3 and~4]{cesa2007secondorder}, using the Bernstein bound of~\citet[Lemma~4]{deRooij2014follow} to handle the Gaussian-perturbed losses.
This results in the following claim.
\begin{restatable}{proposition}{ewvariance}\label{prop:ew-variance}
Let $\cG\subseteq\cF$ be a nonempty finite class, and let $\ell(\hat y,y)=-y\hat y$ be the binary loss.
For every data sequence, the predictor $\tilde f_t$ defined in~\eqref{eq:ew} satisfies the following inequality,
\begin{equation}\label{eq:ew-variance-regret}
\sum_{t=1}^T\En_{G^{(t)}} \ell(\tilde f_t(X_t),Y_t) -\min_{f\in\cG}\sum_{t=1}^T\ell(f(X_t),Y_t) \lesssim \log|\cG|+ \sum_{t=1}^T\En_{G^{(t)}} \var_{f\sim w_t(\cdot\mid G^{(t)})}(f(X_t)).
\end{equation}
\end{restatable}
By~\cref{prop:cover}, the comparator losses over the entire class $\cF$ and over the cover $\cG$ are close, so it remains to bound the cumulative variance on the right-hand side of~\eqref{eq:ew-variance-regret}.
We now show that the Gaussian perturbations alone keep these variances small on the covariates from the past, even though the learning rate is constant.
\begin{restatable}{lemma}{historicalvariance}\label{lemma:historical-variance}
Let $\cG$ be a nonempty finite class of functions $\cX\to[-1,1]$ and let $\ell$ be any loss.
For every $t\le T$ and every fixed history $(X_s,Y_s)_{s<t}$, the weights in~\eqref{eq:ew} satisfy
\begin{equation}\label{eq:empirical-variance}
\sum_{s<t}\left( \En_{G^{(t)}}\var_{f\sim w_t(\cdot\mid G^{(t)})}(f(X_s))\right)^2 \lesssim \log(2|\cG|).
\end{equation}
\end{restatable}
The proof, given in Appendix~\ref{app:ew-proof}, combines Stein's lemma~(\cref{lemma:stein}) with a Gaussian maximal inequality.
Finally, we transfer this bound to the current covariates.
For binary predictions, the variance can be written as
\begin{equation*}
\frac12\En_{G^{(t)}}\var_{f\sim w_t(\cdot\mid G^{(t)})}(f(x)) =\En_{G^{(t)}}\sum_{f,g\in\cG}w_t(f\mid G^{(t)})w_t(g\mid G^{(t)}) \ind\{f(x)\ne g(x)\}.
\end{equation*}
This is a predictable convex mixture over the disagreement class $\{\ind\{f\ne g\}:f,g\in\cF\}$, whose VC dimension is $O(d)$ (see Appendix~\ref{app:ew-proof}).
Applying~\cref{lemma:surprise} to this mixture and summing~\eqref{eq:empirical-variance} over $t\le T$ gives
\[
\sum_{t=1}^T\En \var_{f\sim w_t(\cdot\mid G^{(t)})}(f(X_t)) \lesssim \sqrt{\frac{T\log|\cG|\log(2T)}{\sigma}} +\sqrt{\frac{(d+1)T}{\sigma}} \log^2(2T/\sigma).
\]
Substituting into~\eqref{eq:ew-variance-regret} and adding the cover approximation cost gives~\cref{lemma:ew}.
Finally,~\cref{thm:main} follows by combining~\cref{lemma:reduction,lemma:ew}.

\section{Conclusion and Discussion}

Our results show oracle-efficient learning from smooth data is possible even when the data are agnostic and the base measure unknown.
Moreover, the algorithm that we propose is close to the best that could be hoped for: entirely parameter free, requiring no knowledge of the base measure, the smoothness parameter, or the horizon, and conceptually simple.
Modulo the remaining $\sqrt d$ gap between that which we have shown is achievable by our algorithm and the statistical optimum, our work essentially closes the question of oracle-efficient learning from smooth data.

\paragraph{Concurrent work.}
In independent and concurrent work, \citet{buzaglo2026oracle} establish an oracle-efficient regret bound of $\widetilde O(\sqrt{dT})$ for binary classification with i.i.d. covariates and adversarial labels (which corresponds to the $\sigma=1$ case of our setting).
They analyze the same form of Gaussian FTPL, with a horizon-dependent perturbation scale.
Their analysis controls the mass of disagreement regions of nearly leading hypotheses (that is, the sets of covariates on which at least two near-leaders disagree) through a generalization argument that relies heavily on the i.i.d. assumption. General $\sigma$-smoothed processes require additional care, since it still permits intricate dependencies between covariates, which prevents such a generalization argument to extend directly.
A natural attempt would be to apply our surprise lemma (\cref{lemma:surprise}) to the disagreement indicators.
However, the family of disagreement regions can have infinite VC dimension even when the underlying hypothesis class has VC dimension one (e.g., for the ``star class''~\citep{hanneke2024star}).

\section*{Statement on the use of AI}

In this work, we used generative AI tools to help develop theoretical arguments, formulate and refine mathematical claims and hypotheses, propose proof ingredients, write and revise proofs, and interpret theoretical results.
Over July and August of 2026, the authors established a $\widetilde O(T^{3/4}d^{1/4}/\sqrt{\sigma})$ regret bound for an adaptation of the algorithm of~\cite{wu2024oracle}, extending its original guarantee for i.i.d. covariates ($\sigma=1$) to the general $\sigma$-smooth setting.
The proof analyzed the stability of Rademacher potentials with perturbations on subsampled past covariates, using an information-theoretic decoupling argument to control the effect of dependencies among covariates. The proof was developed and refined through iterative work between the authors and GPT 5.6 Sol.
On September 6, 2026, GPT 6 Astra subsequently proposed the modification to Gaussian FTPL used in the present manuscript.
The corresponding proof was completed through further iterations with the model.
We additionally used GPT 6 Astra, GPT 5.6 Sol, and Claude Opus 5.5 to identify and compare related literature, draft and edit manuscript text, improve exposition and organization, and edit LaTeX.
All mathematical claims in the present manuscript were independently verified by the authors.
The authors take responsibility for the final manuscript, including all text, claims, and proofs produced with AI assistance.

\section*{Acknowledgments}

The authors acknowledge support from the NSF through awards DMS-2031883 and PHY-2019786, the DARPA AIQ program, and AFOSR FA9550-25-1-0375. SV acknowledges support from Amazon AI Research Innovation Fellowship.

\bibliographystyle{plainnat}
\bibliography{refs}
\clearpage

\appendix

\section{Additional preliminaries}\label{app:additional-preliminaries}

We collect preliminaries for smoothed online learning, variance-based regret bounds for exponential weights, Gaussian integration by parts, and Maurey's empirical method.

\subsection{Covers under smoothed processes}\label{app:cover-localization}

VC and fat-shattering dimensions are intimately related to covers.
In a proposition below, we collect bounds that control sizes of covers in terms of class's VC or fat-shattering dimension.
The bounds on projections follow from~\cite{sauer1972density,shelah1972combinatorial,alon1997scale}.
$L^2(\mu)$ versions follow from~\cite{haussler1995sphere,mendelson2003entropy}.
\begin{proposition}\label{prop:classical-covers}
There are universal constants $0<c<1$ and $C>1$ such that the following hold for every $\cF \colon \cX \to [-1,1]$, and $\eps \in (0,1]$.
\begin{enumerate}
\item[(i)] For any set of $n$ points $z_{1:n}$, there exist a finite cover $\cG \subset \cF$ such that:
\[
\log|\cG|\le C \fat_{c\eps}(\cF) \log^2(Cn/\eps), \qquad \sup_{f\in\cF}\min_{g\in\cG}\max_{i\in[n]}|f(z_i)-g(z_i)|\le\eps.
\]
If $\cF$ is binary with VC dimension $d$, the cover can be chosen with $\log |\cG|\le d\log (Cn)$ and, for every $f\in\cF$, there exists $g\in\cG$ such that $f(z_i)=g(z_i)$ for all $i\in[n]$.

\item[(ii)] For every probability measure $\mu$ on $\cX$, there is a finite subset $\cG\subset \cF$ such that
\[
\log|\cG|\le C\fat_{c\eps}(\cF)\log(C/\eps), \qquad \sup_{f\in\cF}\min_{g\in\cG}\|f-g\|_{L^2(\mu)}\le\eps.
\]
If $\cF$ is binary, we have $\log|\cG|\le Cd\log(C/\eps)$ and the same approximation guarantee holds.
\end{enumerate}
\end{proposition}
It is well-established in the smoothed online learning literature that covers w.r.t. $L^1(\mu)$ norm provide approximation guarantees along the draw of a smoothed trajectory~\citep{haghtalab2024smoothed}.
In particular, the following proposition follows using the coupling construction of~\citet[Theorem~2.1 and Appendix~B.2]{haghtalab2024smoothed}.

\coverguarantee*

\begin{proof}[Proof of~\cref{prop:cover}]
Set $\delta=1/(2T)$ and $n=\lceil\sigma^{-1}\log(T/\delta)\rceil$.
The recursive coupling construction of~\citet[Theorem~2.1 and Appendix~B.2]{haghtalab2024smoothed} implies existence of a coupling between the smooth covariate sequence $X_{1:T}$ and an array $Z_{1:T,1:n}$, drawn i.i.d. from $\mu$, so that
\begin{equation}\label{eq:coupling}
\P\left[X_t\in\{Z_{t,1},\ldots,Z_{t,n}\}\text{ for every }t\le T\right] \ge 1-T(1-\sigma)^n\ge1-\delta.
\end{equation}
On the displayed event, for any pair of functions $f,g$, we can upper bound the discrepancy between $f$ and $g$ on $X_{1:T}$ as:
\[
\sum_{t=1}^T\ind\{f(X_t)\ne g(X_t)\} \le\sum_{t=1}^T\sum_{i=1}^n\ind\{f(Z_{t,i})\ne g(Z_{t,i})\}.
\]
From part~(ii) of~\cref{prop:classical-covers}, there exist a cover $\cG \subset \cF$ with the following property:
\begin{equation}\label{eq:cov-guarantee}
\log|\cG|\lesssim d\log(2T/\sigma), \qquad \forall f\in\cF\ \exists g\in\cG:\quad \En_{Z\sim\mu}|f(Z)-g(Z)|\le\frac{\sigma}{T}.
\end{equation}
Now, consider the disagreement class
\[
\cH:=\cF\Delta\cF =\{x\mapsto\ind\{f(x)\ne g(x)\}:f,g\in\cF\}.
\]
The VC dimension of such a class is in $O(d)$~\cite{blumer1989learnability}.
Results for norm comparison inequalities for VC classes~\citep[Exercise~12.4]{boucheron2013concentration} thus imply that, with probability at least $1-\delta$ over the i.i.d. array $Z_{1:T,1:n}$,
\[
\sup_{h\in\cH}\left[ \frac1{Tn}\sum_{t=1}^T\sum_{i=1}^n h(Z_{t,i})-2\En_\mu h \right] \lesssim\frac{d\log(2Tn)+\log(1/\delta)}{Tn}.
\]
Enumerate $\cG=\{g_1,\ldots,g_{|\cG|}\}$ and, for each $f\in\cF$, choose $\pi(f)\in[|\cG|]$ so that $g_{\pi(f)}$ satisfies~\cref{eq:cov-guarantee}.
On the preceding event,
\[
\sup_{f\in\cF}\sum_{t=1}^T\sum_{i=1}^n\ind\{f(Z_{t,i})\ne g_{\pi(f)}(Z_{t,i})\} \lesssim\sigma n+d\log(2Tn)+\log(1/\delta) \lesssim d\log(2T/\sigma),
\]
where the last inequality uses $\sigma n\le\log(2T^2)+1$, $\log(2Tn)\lesssim\log(2T/\sigma)$, and $d\ge1$.
Thus,
\[
\sup_{f\in\cF} \inf_{g\in \cG} \sum_{t=1}^T\sum_{i=1}^n\ind\{f(Z_{t,i})\ne g(Z_{t,i})\} \lesssim\sigma n+d\log(2Tn)+\log(1/\delta) \lesssim d\log(2T/\sigma).
\]
Intersecting this high probability event with one in~\cref{eq:coupling}, we obtain with probability $1- 2\delta = 1-1/T$:
\[
\sup_{f\in\cF} \inf_{g\in \cG} \sum_{t=1}^T \ind\{f(X_t)\ne g(X_t)\} \lesssim\sigma n+d\log(2Tn)+\log(1/\delta) \lesssim d\log(2T/\sigma).
\]
This concludes the proof.
\end{proof}

For real-valued classes, the above can be extended to bounds that depend on the \emph{fat-shattering dimension} of the class.

\begin{proposition}\label{prop:fat-cover}
Let $\cF\subseteq[-1,1]^{\cX}$.
There are universal constants $c,C>0$ such that the following hold.
For every $\mu$, horizon $T\ge1$, $\sigma\in(0,1]$, and $0<\eps\le1/16$, there exist a finite subset $\cG = \{g_1,\ldots,g_{|\cG|}\}\subseteq\cF$ and a map $\pi \colon \cF \to [|\cG|]$ with $g_{\pi(f)}=f$ for $f\in\cG$, such that
\begin{equation}\label{eq:real-cover-size}
\log|\cG|\lesssim \fat_{c\eps}(\cF)\log^2\frac{CT}{\sigma\eps}
\end{equation}
and, for every covariate process that is $\sigma$-smooth with respect to $\mu$, with probability at least $1-1/T$, we have
\begin{equation}\label{eq:real-cover-count}
\sup_{f\in\cF}\sum_{s=1}^T \ind\{|f(X_s)-g_{\pi(f)}(X_s)|>2\eps\} \lesssim\fat_{c\eps}(\cF)\log^2\frac{CT}{\sigma\eps}.
\end{equation}
\end{proposition}

\begin{proof}
We use a double-sampling trick.
First, let
\[
\delta=1/T,\quad n=\left\lceil\sigma^{-1}\log(2T/\delta)\right\rceil,\quad N=Tn,\quad k=\left\lceil C\fat_{c\eps}(\cF) \log^2\frac{CT}{\sigma\eps}\right\rceil.
\]
Let $Z_{1:N}$ and $Z'_{1:N}$ be independent i.i.d.\ samples from $\mu$.
We claim that
\begin{equation}\label{eq:double-sample}
\P\left[ \sup_{\substack{f,g\in\cF,\\
\max_{j\le N}|f(Z'_j)-g(Z'_j)|\le\eps} } \, \sum_{j=1}^N\ind\{|f(Z_j)-g(Z_j)|>2\eps\}\ge k \right]\le\frac\delta2.
\end{equation}
Condition on the unordered multiset $\{Z_j,Z'_j\}, j \in [N]$.
Choose an $\eps/4$-cover $\cG'\subseteq\cF$ on the $2N$ points $\{Z_j,Z'_j\}$, $j\le N$ using~\cref{prop:classical-covers}(i), so that $\cG'$ depends only on the multiset of $2N$ points, but not on their ordering.
By~\cref{prop:classical-covers}(i), its cardinality satisfies
\[
\log|\cG'| \lesssim\fat_{c\eps}(\cF)\log^2(CN/\eps) \lesssim\fat_{c\eps}(\cF)\log^2\frac{CT}{\sigma\eps}.
\]
If $f,g$ satisfy the event in~\eqref{eq:double-sample}, choose functions $h,h'\in\cG'$ from the cover such that $|h-f|\le\eps/4$ and $|h'-g|\le\eps/4$ on all $2N$ points.
Then, on the event~\cref{eq:double-sample}, we have
\[
\max_{j\le N}|h(Z'_j)-h'(Z'_j)|\le\frac{3\eps}2, \qquad \sum_{j=1}^N\ind\left\{|h(Z_j)-h'(Z_j)|>\frac{3\eps}2\right\}\ge k.
\]
Fix $h,h'\in\cG'$.
Now, recall that so far we have conditioned on the multiset of $2N$ observations, but not on their ordering.
Call a sample $Z \in \{Z_j,Z'_j\}_{j\le N}$ good if $|h(Z) - h'(Z)| \le 3\eps/2$, and call it bad otherwise.
Then, on the above event, all bad samples are from $Z_{1:N}$, and there are at least $k$ of them.
By exchangeability, the probability that all $\ge k$ bad samples are assigned to $Z_{1:N}$ is at most $2^{-k}$.
Thus, conditional on the multiset of $2N$ points, we have for any fixed pair $h, h'$:
\[
\P\left[ \max_{j\le N}|h(Z'_j)-h'(Z'_j)|\le\frac{3\eps}2, \quad \sum_{j=1}^N\ind\left\{|h(Z_j)-h'(Z_j)|>\frac{3\eps}2\right\}\ge k \;\middle|\; \{Z_j,Z'_j\}_{j\le N} \right] \le2^{-k}.
\]
A union bound over $h,h'\in\cG'$ therefore bounds the conditional probability in~\eqref{eq:double-sample} by
\[
|\cG'|^2 2^{-k} =\exp\left(2\log|\cG'|-k\log2\right) \le\frac\delta2,
\]
where the last inequality follows from the cardinality bound and $\log(2/\delta)\lesssim\log(2T)$, for $C$ sufficiently large.
This proves~\eqref{eq:double-sample}.

Now, to construct the desired cover, note that~\cref{eq:double-sample} implies existence of at least one \emph{fixed} tuple $z_{1:N}'$ such that, when $Z_{1:N} \sim \mu$, we have
\[
\P \left[\sup_{\substack{f,g\in\cF,\\
\max_{j\le N}|f(z'_j)-g(z'_j)|\le\eps} } \, \sum_{j=1}^N\ind\{|f(Z_j)-g(Z_j)|>2\eps\}\ge k \right]\le\frac\delta2.
\]
Construct an $\eps$-cover $\cG$ on $z'_{1:N}$ using~\cref{prop:classical-covers}(i), and choose, for each $f\in\cF$, a cover element $g_{\pi(f)}\in\cG$ such that
\[
\max_{j\le N}|f(z'_j)-g_{\pi(f)}(z'_j)|\le\eps,
\]
taking $g_{\pi(f)}=f$ for $f\in\cG$.
\Cref{prop:classical-covers}(i) gives~\eqref{eq:real-cover-size}.
By~\eqref{eq:double-sample}, with probability at least $1-\delta/2$ over a fresh i.i.d.\ sample $Z_{1:N}$,
\[
\sup_{f\in\cF}\sum_{j=1}^N \ind\{|f(Z_j)-g_{\pi(f)}(Z_j)|>2\eps\}<k.
\]
The conclusion follows from~\eqref{eq:coupling}: with probability at least $1-\delta/2$, we can couple $X_{1:T}$ to $Z_{1:N}$ such that $X_{1:T} \subset Z_{1:N}$.
Applying union bound gives:
\[
\sup_{f\in\cF}\sum_{s=1}^T \ind\{|f(X_s)-g_{\pi(f)}(X_s)|>2\eps\} \lesssim \fat_{c\eps}(\cF)\log^2\frac{CT}{\sigma\eps},
\]
with probability $1- \delta$, as desired.

\end{proof}

\subsection{Variance-based regret bounds}\label{app:ew-variance}

Variance-based bounds are a standard tool for analyzing exponential weights~\citep{cesa2007secondorder}.
The following is a one-round Bernstein bound of~\citet[Lemma~4]{deRooij2014follow}, which will be useful in our analysis of exponential weights.
\begin{proposition}\label{prop:ew-variance-general}
Let $\cG$ be a nonempty finite class, let $w\in\Delta(\cG)$ be a set of weights, and let $\ell\colon\cG\to\bbR$ be a loss function.
Suppose,
\[
\max_{f\in\cG}\ell(f)-\min_{f\in\cG}\ell(f)\le s,
\]
that is, $s$ upper bounds the range of the loss.
Then, for every $\eta>0$,
\begin{align*}
\En_{f\sim w}\ell(f) +\frac1\eta\log\sum_{f\in\cG}w(f)\exp\left(-\eta\ell(f)\right) &\le \frac{\exp\left(\eta s\right)-\eta s-1}{\eta s^2} \var_{f\sim w}(\ell(f)),
\end{align*}
where the coefficient is interpreted as $\eta/2$ when $s=0$.
\end{proposition}

\subsection{Gaussian geometry}

The following identity is due to~\cite{stein1981estimation}.
Intuitively, it will allow us to relate gradients of certain potentials to their correlations with the Gaussians.
\begin{lemma}[Stein's lemma]\label{lemma:stein}
Let $G=(G_1,\ldots,G_n)$ have independent standard Gaussian coordinates, and let $\Psi:\mathbb R^n\to\mathbb R$ be Lipschitz.
Then, for every $s\le n$,
\[
\En_G\partial_{G_s}\Psi(G)=\En_G[G_s \Psi(G)].
\]
\end{lemma}
To control the aforementioned correlations, the following two assertions will be useful.
\begin{lemma}\label{lemma:gaussian-correlation}
Let $G=(G_1,\ldots,G_n)$ have independent standard Gaussian coordinates.
\begin{enumerate}
\item[(i)] For every potential $\Psi:\bbR^n\to\bbR$ with $\En_G[\Psi(G)^2]<\infty$, we have $\sum_{s\le n}\left(\En_G[G_s\Psi(G)]\right)^2\le\var_G(\Psi(G)).$
\item[(ii)] Consider a finite class $\cH \colon \cX \to [-1,1]$ and any $h\colon \cX \times \bbR^n \to [-1,1]$ such that $h(\cdot, G) \in \conv(\cH)$ almost surely.
Then, for any $z_{1:n} \in \cX^n$
\[
\sum_{i\le n}\left(\En_G[G_i h(z_i, G)]\right)^2\le2\log |\cH|.
\]
\end{enumerate}
\end{lemma}
\begin{proof}
First, note that we may assume without loss of generality that $\En \Psi(G) = 0$.
Then, (i) can be recognized as an instance of Bessel inequality for $L^2$ space.
Indeed, let $\gamma$ be the standard Gaussian measure on $\bbR$.
For $f,g\in L^2(\gamma^n)$, let
\[
\langle f, g\rangle_{L^2(\gamma^n)} := \En_{G \sim \gamma^n} f(G)g(G).
\]
Let $e_i \colon x \mapsto x_i$ be the $i$th coordinate function; then, $\{e_i\}$ are orthonormal in this space, and we have
\begin{align*}
\sum_{s\le n}\left(\En_G[G_s\Psi(G)]\right)^2 &= \sum_{s\le n}\left(\langle e_s, \Psi \rangle_{L^2(\gamma^n)}\right)^2 \\
&\le \|\Psi\|_{L^2(\gamma^n)}^2 \\
&= \En_G\Psi(G)^2 = \var_G(\Psi(G)),
\end{align*}
where the inequality is Bessel's inequality and the last equality uses $\En_G\Psi(G)=0$.
This concludes the proof of part~(i).

For~(ii), let $\alpha_s=\En_G[G_sh(z_s,G)]$ for $s\le n$, and assume that $\alpha\ne0$, as otherwise there is nothing to prove.
For a fixed $G$, the map $h\mapsto\sum_{s\le n}\alpha_sG_sh(z_s)$ is linear, so its value at $h(\cdot,G)\in\conv(\cH)$ is at most its maximum over $\cH$.
Thus,
\[
\|\alpha\|_2^2 =\En_G\sum_{s\le n}\alpha_sG_sh(z_s,G) \le\En_G\max_{h\in\cH}\sum_{s\le n}\alpha_sG_sh(z_s).
\]
For each $h\in\cH$, the sum inside the maximum is a centered Gaussian with variance $\sum_{s\le n}\alpha_s^2h(z_s)^2\le\|\alpha\|_2^2$.
Thus, for every $\lambda>0$, Jensen's inequality and the Gaussian moment-generating function give
\begin{align*}
\En_G\max_{h\in\cH}\sum_{s\le n}\alpha_sG_sh(z_s) &\le\frac1\lambda\log\sum_{h\in\cH} \En_G\exp\left(\lambda\sum_{s\le n}\alpha_sG_sh(z_s)\right) \\
&\le\frac{\log|\cH|}{\lambda}+\frac{\lambda\|\alpha\|_2^2}{2}.
\end{align*}
Taking $\lambda=\sqrt{2\log|\cH|}/\|\alpha\|_2$ gives $\|\alpha\|_2^2\le\sqrt{2\log|\cH|}\,\|\alpha\|_2$.
Dividing by $\|\alpha\|_2$ and squaring concludes the proof.
\end{proof}

\subsection{Convex hull approximations} \label{app:maurey}

We use the following consequence of Maurey's empirical method~\citep{pisier1981remarques}, in the $\ell_\infty$ form due to~\citet{althofer1994sparse}.

\begin{proposition}[Maurey's empirical method]\label{prop:maurey}
Let $\cF\subseteq[-1,1]^{\cX}$.
For any $h\in\conv(\pm\cF)$, integers $m,K\ge1$, and points $Z_1,\ldots,Z_m\in\cX$, there exist functions $f_1,\ldots,f_K\in\pm\cF$ such that
\[
\max_{s\in[m]}\left|h(Z_s)-\frac1K\sum_{i=1}^K f_i(Z_s)\right| \le\sqrt{\frac{2\log(2m)}{K}}.
\]
\end{proposition}

\subsection{Log-partition functions and their derivatives}

\begin{lemma}[Derivatives of log-partition functions]\label{lemma:log-partition}
Let $\psi_1,\ldots,\psi_N:\mathbb R^n\to\mathbb R$ be Lipschitz, and define
\[
\Psi(\theta)=\log\sum_{i=1}^N \exp\left(\psi_i(\theta)\right), \qquad w_i(\theta)=\frac{\exp\left(\psi_i(\theta)\right)}{\sum_{j=1}^N \exp\left(\psi_j(\theta)\right)}.
\]
Then, for almost every $\theta\in\mathbb R^n$, every $r\le n$, and every $h\in\mathbb R^N$,
\begin{align*}
\partial_{\theta_r}\Psi(\theta) =\En_{i\sim w(\theta)}\partial_{\theta_r}\psi_i(\theta), \qquad
\partial_{\theta_r}\En_{i\sim w(\theta)}h_i =\cov_{i\sim w(\theta)}\left(h_i,\partial_{\theta_r}\psi_i(\theta)\right).
\end{align*}
In particular, if $\psi_i(\theta)=a_i+\langle\theta,z_i\rangle$ for some $a_i\in\mathbb R$ and $z_i\in\mathbb R^n$, then for every $\theta\in\mathbb R^n$,
\begin{align*}
\partial_{\theta_r}\Psi(\theta) =\En_{i\sim w(\theta)}(z_i)_r, \qquad
\partial_{\theta_r}\En_{i\sim w(\theta)}h_i =\cov_{i\sim w(\theta)}\left(h_i,(z_i)_r\right).
\end{align*}
\end{lemma}

\begin{proof}
By Rademacher's theorem, almost every $\theta$ is a point of differentiability of all $\psi_1,\ldots,\psi_N$.
At such $\theta$, the chain rule gives the first identity and
\[
\partial_{\theta_r}w_i(\theta) =w_i(\theta)\left(\partial_{\theta_r}\psi_i(\theta) -\En_{j\sim w(\theta)}\partial_{\theta_r}\psi_j(\theta)\right).
\]
Multiplying by $h_i$ and summing over $i$ gives the second identity.
In the linear case, each $\psi_i$ is differentiable everywhere with $\partial_{\theta_r}\psi_i(\theta)=(z_i)_r$, so substituting into the two identities gives the stated result.
\end{proof}

\section{Proofs for the binary case}\label{app:binary-proofs}

The binary proof has three ingredients: a surprise lemma using second moments, a reduction to exponential weights, and a regret bound for the exponential-weights reference.

\subsection{Surprise lemma}\label{app:surprise}

We first bound the cumulative expected values of predictable test functions by their historical second moments on a tangent sequence.
A norm comparison then replaces the tangent moments by their realized counterparts.

\subsubsection{Surprise lemma with tangent sequence}\label{app:surprise-tangent}

The following lemma follows from the change-of-measure argument of~\citet[Section~3.2, equations~(3)--(5)]{xie2022coverage}.
We restate the result in our notation, and provide a proof for completeness.

\begin{restatable}{lemma}{densitytransfer}\label{lemma:density-transfer}
For predictable test functions $h_t\colon\cX\to[0,1]$, we have
\begin{equation}\label{eq:density-transfer}
\En\sum_{t=1}^T h_t(X_t) \lesssim \sqrt{\frac{\log(2T)}\sigma \En\sum_{t=1}^T\sum_{s<t}h_t(\bar X_s)^2} +\frac{\log(2T)}\sigma.
\end{equation}
\end{restatable}

\begin{proof}
We begin by letting
\[
a_t(x)=\sigma \frac{\mathrm dp_t}{\mathrm d\mu}(x)\in[0,1].
\]
Similar to~\cite{xie2022coverage}, for each $x \in \cX$ let
\[
\tau(x) := \min\left\{t\in[T] \colon \sum_{s \le t} a_s(x) \ge 1\right\}
\]
be the index of the first round in which the cumulative scaled density at $x$ reaches $1$, with $\tau(x)=T$ if the set is empty.
Intuitively, we can partition the surprise contribution $\En h_t(X_t)$ in each round into $x$'s that are still in the ``burn-in'' period ($t\le\tau(x)$) and those that have left the burn-in phase.
Observe that the burn-in contribution can be upper bounded as follows:
\begin{equation}
\frac1\sigma\int\sum_{t=1}^{\tau(x)} a_t(x)h_t(x) \dd\mu(x) \le \frac{1}{\sigma} + \frac1\sigma\int\sum_{t=1}^{\tau(x)-1} a_t(x)h_t(x) \dd\mu(x) \le \frac2\sigma, \label{eq:burn-in}
\end{equation}
where we used $|h_t| \le 1$.
It remains to bound the contribution of those covariates that have left the burn-in period.
From \citet[Lemma~4]{xie2022coverage}, we have with probability $1$:
\[
\sum_{t=1}^T\frac{a_t(x)}{1+\sum_{s<t}a_s(x)} \lesssim \log(2T).
\]
On the event $\{\sum_{s<t}a_s(x)\ge1\}$, we also have
\[
\frac{a_t(x)^2}{\sum_{s<t}a_s(x)} \le\frac{2a_t(x)}{1+\sum_{s<t}a_s(x)}.
\]
By Cauchy--Schwarz, the contribution after burn-in is bounded by
\begin{align}
&\frac1\sigma\En\int\sum_{t=\tau(x)+1}^T a_t(x)h_t(x)\dd\mu(x) \notag \\
&\quad \le \left(\frac1\sigma\En\int\sum_{t=\tau(x)+1}^T \frac{a_t(x)^2}{\sum_{s<t}a_s(x)}\dd\mu(x)\right)^{1/2} \left(\frac1\sigma\En\int\sum_{t=1}^T \left(\sum_{s<t}a_s(x)\right)h_t(x)^2\dd\mu(x)\right)^{1/2} \notag \\
&\quad \lesssim \sqrt{\frac{\log(2T)}\sigma \frac1\sigma\En\int\sum_{t=1}^T \left(\sum_{s<t}a_s(x)\right)h_t(x)^2\dd\mu(x)} \notag \\
&\quad = \sqrt{\frac{\log(2T)}\sigma \En\sum_{t=1}^T\sum_{s<t}\En_{p_s}h_t^2}. \notag
\end{align}
Combining the above with~\cref{eq:burn-in}, we obtain:
\[
\En\sum_{t=1}^T h_t(X_t) \lesssim \sqrt{\frac{\log(2T)}\sigma \En\sum_{t=1}^T\sum_{s<t}h_t(\bar X_s)^2} +\frac{\log(2T)}\sigma,
\]
as desired.
\end{proof}
For the same predictable functions $0\le h_t\le1$, \citet[Lemma~12]{block2024performance} gives, up to universal constants,
\[
\En\sum_{t=1}^T h_t(X_t) \lesssim \frac{\log(2T/\sigma)}\sigma +\log(2T/\sigma) \sqrt{\frac T\sigma \En\sum_{t=1}^T\frac1t\sum_{s<t}\En_{p_s}h_t}.
\]
Since $h_t^2\le h_t$ and $t\le T$,
\[
\En\sum_{t=1}^T\sum_{s<t}\En_{p_s}h_t^2 \le T\En\sum_{t=1}^T\frac1t\sum_{s<t}\En_{p_s}h_t.
\]
Thus~\eqref{eq:density-transfer} also implies this first-moment bound in~\cite{block2024performance}.

\subsubsection{Comparison inequality for signed convex hulls}

We next compare tangent and realized second moments uniformly for elements of signed convex hulls.

\begin{restatable}{lemma}{tangentcomparison}\label{lemma:mixture-norm}
Under the same assumptions as in~\cref{lemma:surprise}, we have
\[
\En\sup_{h\in\conv(\pm\cF)} \left\{\sum_{t\le T}\En_{\bar X_t} h^2(\bar X_t) - 12\sum_{{t\le T}} h(X_t)^2\right\} \lesssim d\log^3(2T/\sigma),
\]
\end{restatable}

\begin{proof}[Proof of~\cref{lemma:mixture-norm}]
In the proof, we first reduce the claim to comparison for binary function class using layer-cake identity for expectations and Maurey's empirical method (\cref{prop:maurey}).
Then, the proof is concluded by invoking~\citet[Theorem~10]{block2024performance}.
To start, using the layer-cake identity, we have
\[
\sum_{s=1}^T \En_{\bar X_s} h^2(\bar X_s) \le 1 + \int_{1/\sqrt{T}}^1 2\lambda \sum_{s=1}^T \P[|h(\bar X_s)|\ge \lambda]\dd \lambda.
\]
Also, note that we have:
\[
\sum_{s=1}^T h(X_s)^2 \ge \int_{1/\sqrt{T}}^1 \frac{\lambda}{2}\sum_{s=1}^T \ind\{|h(X_s)|\ge \lambda/2\}\dd \lambda.
\]
Thus,
\begin{align}
&\sum_{s=1}^T\En_{\bar X_s}h^2(\bar X_s) -12\sum_{s=1}^T h(X_s)^2 \notag \\
&\quad \le1+\int_{1/\sqrt T}^1 2\lambda\, \left\{\sum_{s=1}^T\P[|h(\bar X_s)|\ge\lambda] -3\sum_{s=1}^T\ind\{|h(X_s)|\ge\lambda/2\}\right\}\dd\lambda. \label{eq:layer-cake}
\end{align}
For $1/\sqrt T\le\lambda\le1$, take $K=\lceil32\lambda^{-2}\log(4T)\rceil$, and define
\[
\cF_\lambda=\left\{ x\mapsto\ind\left\{\left|\frac1K\sum_{i=1}^K f_i(x)\right| \ge\frac{3\lambda}{4}\right\}: f_1,\ldots,f_K\in\pm\cF\right\},
\]
to be the class obtained by averaging $K$ functions from $\cF$.
Applying~\cref{prop:maurey} to the $2T$ points $X_{1:T},\bar X_{1:T}$, for every $h\in\conv(\pm\cF)$ there exists $g\in\cF_\lambda$ such that
\begin{equation*}
\ind\{|h(x)|\ge\lambda\}\le g(x)\le\ind\{|h(x)|\ge\lambda/2\}
\end{equation*}
for every $x\in\{X_s,\bar X_s:s\le T\}$.
Set $m=\lceil2T\log(T)/\sigma\rceil$ and let $Z_{1:m}\sim\mu^m$.
Using \citet[Theorem~10]{block2024performance} and a bound on Wills functional using growth function of the class, we have:
\begin{align}
&\En\sup_{h\in\conv(\pm\cF)} \left\{\sum_{s=1}^T\P[|h(\bar X_s)|\ge\lambda] -3\sum_{s=1}^T\ind\{|h(X_s)|\ge\lambda/2\}\right\} \notag \\
&\quad \le\En \sup_{h\in\conv(\pm\cF)} \left\{\sum_{s=1}^T\ind\{|h(\bar X_s)|\ge\lambda\} -3\sum_{s=1}^T\ind\{|h(X_s)|\ge\lambda/2\}\right\} \notag \\
&\quad \le\En\sup_{g\in\cF_\lambda} \sum_{s=1}^T\left(g(\bar X_s)-3g(X_s)\right) \notag \\
&\quad \lesssim 1+\log\sup_{Z_{1:m}}\left|\left.\cF_\lambda\right|_{Z_{1:m}}\right|. \label{eq:growth-fn-bound}
\end{align}
We now bound the growth function of $\cF_{\lambda}$.
On any $m$ points $Z_1,\ldots,Z_m$, the Sauer--Shelah lemma (part~(i) of~\cref{prop:classical-covers}) gives
\[
\log \left|\left.\pm\cF\right\vert_{Z_{1:m}}\right| \le C d \log (Cm).
\]
Thus, via a simple counting argument, we have:
\[
\log \left|\left.\cF_\lambda\right\vert_{Z_{1:m}}\right| \le C K d \log (Cm).
\]
Thus, we have
\[
1+\log\sup_{Z_{1:m}}\left|\left.\cF_\lambda\right|_{Z_{1:m}}\right| \lesssim dK\log(2m) \lesssim d\lambda^{-2}\log^2(2T/\sigma).
\]
Plugging this into~\cref{eq:growth-fn-bound,eq:layer-cake} gives
\begin{align*}
&\En\sup_{h\in\conv(\pm\cF)} \left\{\sum_{s=1}^T\En_{\bar X_s}h^2(\bar X_s) -12\sum_{s=1}^T h(X_s)^2\right\} \\
&\quad \lesssim 1+\int_{1/\sqrt T}^1 2\lambda\left(d\lambda^{-2}\log^2(2T/\sigma)\right)\dd\lambda \lesssim d\log^3(2T/\sigma),
\end{align*}
as desired.
\end{proof}

\subsubsection{Proof of the surprise lemma}

It remains to combine the change-of-measure bound in~\cref{lemma:density-transfer} with the uniform norm comparison in~\cref{lemma:mixture-norm}.

\surpriselemma*

\begin{proof}[Proof of~\cref{lemma:surprise}]
If $\sigma T<1$, the right-hand side of~\eqref{eq:surprise} is at least a constant multiple of $T$, so the claim follows from $|h_t|\le1$.
Otherwise, apply \cref{lemma:density-transfer} to $|h_t|$ and the uniform bound in \cref{lemma:mixture-norm} to $h_t$ on each prefix $X_{1:t-1}$.
This gives
\begin{align*}
\En\sum_{t=1}^T|h_t(X_t)| &\lesssim \sqrt{\frac{\log(2T)}\sigma \left(\En\sum_{t=1}^T\sum_{s<t}h_t(X_s)^2 +dT\log^3(2T/\sigma)\right)} +\frac{\log(2T)}\sigma \\
&\lesssim \sqrt{\frac{\log(2T)}\sigma \En\sum_{t=1}^T\sum_{s<t}h_t(X_s)^2} +\sqrt{\frac{dT}\sigma}\log^2(2T/\sigma),
\end{align*}
where the last step uses $\sigma T\ge1$, $d\ge1$, and $\log(2T)\le\log(2T/\sigma)$.
\end{proof}

The proofs of~\cref{lemma:mixture-norm,lemma:surprise} use the VC dimension only to count the restrictions of $\pm\cF$ to finitely many points.
The same argument therefore applies to finite classes of real-valued functions, which we use in Appendix~\ref{app:real-proofs}.

\begin{corollary}\label[corollary]{cor:finite-surprise}
Let $\cF\subseteq[-1,1]^{\cX}$ be a nonempty finite class, and let $\{h_t\}_{t\in[T]}$ be a predictable sequence of functions with $h_t\in\conv(\pm\cF)$.
Then,
\[
\En\sum_{t=1}^T|h_t(X_t)| \lesssim \sqrt{\frac{\log(2T)}\sigma \En\sum_{t=1}^T\sum_{s<t}h_t(X_s)^2}+\sqrt{\frac{\log(2|\cF|)T}\sigma}\log^2\left(\frac{2T}{\sigma}\right).
\]
\end{corollary}

\subsection{Comparing predictions on the past}\label{app:prediction}

We begin by proving that predictions of FTPL~\cref{eq:ftpl} and exponential weights~\cref{eq:ew} in round $t$ are close on past covariates $X_{s}$ with $s < t$.

\empiricalprediction*

\begin{proof}[Proof of~\cref{lemma:empirical-prediction}]
Fix the history $(X_{1:t-1},Y_{1:t-1})$ and let $G:=G^{(t)}$ to lighten the notational load.
Use the same potentials as in the proof sketch:
\begin{align*}
\Phi_t(G) &=\max_{f\in\cF}\sum_{s<t}(Y_s-G_s)f(X_s), \\
\widetilde\Phi_t(G) &=\log\sum_{f\in\cG} \exp\left(\sum_{s<t}(Y_s-G_s)f(X_s)\right) \\
&= \max_{w\in\Delta(\cG)}\left[\En_{f \sim w} \sum_{s<t}(Y_s-G_s)f(X_s) + H(w)\right].
\end{align*}
$w_t(\cdot\mid G)$ from~\eqref{eq:ew} maximize the latter potential in its variational form.
Note that, by absolute continuity of the Gaussians, the set of all functions that maximize the former or the latter potential share predictions on every past covariate $X_s$.
Thus, for every $s<t$ and almost every $G$, we may apply Danskin's theorem for the finite maximum \citep[Appendix~B]{bertsekas2016nonlinear}.
Combining it with~\cref{lemma:log-partition}, differentiating the two potentials gives:
\[
\partial_{G_s}\Phi_t(G)=-\hat f_t(X_s),\qquad \partial_{G_s}\widetilde\Phi_t(G)=-\tilde f_t(X_s).
\]
Both potentials are Lipschitz functions of $G$, so Stein's lemma (\cref{lemma:stein}) yields
\begin{align*}
\sum_{s<t} \left(\En_G\hat f_t(X_s)-\En_G\tilde f_t(X_s)\right)^2 &=\left\|\En_G\nabla_G (\Phi_t(G)-\widetilde\Phi_t(G))\right\|_2^2 =\left\|\En_G[G(\Phi_t(G)-\widetilde\Phi_t(G))]\right\|_2^2.
\end{align*}
Applying part~(i) of~\cref{lemma:gaussian-correlation} further gives:
\begin{align}
\sum_{s<t} \left(\En_G\hat f_t(X_s)-\En_G\tilde f_t(X_s)\right)^2 &\le\En_G(\Phi_t-\widetilde\Phi_t)^2. \label{eq:pot-diff-ub}
\end{align}
Now, let
\[
\Phi_t^\dagger(G) =\max_{f\in\cG}\sum_{s<t}(Y_s-G_s)f(X_s),
\]
be the unregularized potential over the cover $\cG$.
Since entropy is upper bounded by $\log |\cG|$, we have:
\begin{equation}\label{eq:ent-bound}
0\le\widetilde\Phi_t(G)-\Phi_t^\dagger(G) \le\log|\cG|\lesssim d\log(2T/\sigma).
\end{equation}
By~\cref{prop:cover}, with probability at least $1-1/T$ over the history,
\[
\sup_{f\in\cF}\inf_{g\in\cG} \sum_{s<t}\ind\{f(X_s)\ne g(X_s)\} \le C d\log(2T/\sigma).
\]
On the above event, using $Y_t \in \{\pm1\}$, we have:
\begin{align}
0\le\Phi_t(G)-\Phi_t^\dagger(G) &\le 2C d\log(2T/\sigma) \notag \\
&\qquad +\max_{\substack{f\in\cF,\ g\in\cG:\\
\sum_{s<t}\ind\{f(X_s)\ne g(X_s)\} \le C d\log(2T/\sigma)}} \left|\sum_{s<t}G_s(f(X_s)-g(X_s))\right|. \label{eq:ub-gaussian-max}
\end{align}
For each pair in the maximum, its Gaussian sum has variance
\[
\var_G\left( \sum_{s<t}G_s(f(X_s)-g(X_s))\right) =4\sum_{s<t}\ind\{f(X_s)\ne g(X_s)\} \le4C d\log(2T/\sigma).
\]
By~\cite{blumer1989learnability}, the difference class has VC dimension in $O(d)$.
Thus, by Sauer--Shelah lemma (\cref{prop:classical-covers}), the maximum in~\cref{eq:ub-gaussian-max} is taken over at most $\exp(C d \log (Ct))$ functions.
Applying Gaussian tail inequality with a union bound then gives:
\begin{align*}
&\En_G\left[ \max_{\substack{f\in\cF,\ g\in\cG:\\
\sum_{s<t}\ind\{f(X_s)\ne g(X_s)\} \le C d\log(2T/\sigma)}} \left|\sum_{s<t}G_s(f(X_s)-g(X_s))\right|^2\right] \\
&\quad \lesssim d\log(2T/\sigma)\cdot d\log(2t) \lesssim d^2\log^2(2T/\sigma).
\end{align*}
It follows that $\En_G(\Phi_t-\Phi_t^\dagger)^2\lesssim d^2\log^2(2T/\sigma)$.
Combining this with~\cref{eq:ent-bound}, we have:
\begin{align*}
\En_G(\Phi_t-\widetilde\Phi_t)^2 &\le\En_G(\Phi_t-\Phi_t^\dagger)^2 +\En_G(\widetilde\Phi_t-\Phi_t^\dagger)^2 \\
&\lesssim d^2\log^2(2T/\sigma).
\end{align*}
Noting that the failure event of~\cref{prop:cover} contributes at most $O(1)$ to the expected squared prediction discrepancy concludes the proof.
\end{proof}

\subsection{Proof of the reduction to exponential weights}\label{app:reduction}

We now apply the surprise lemma~\cref{lemma:surprise} to the difference of predictions.

\ewreduction*

\begin{proof}[Proof of~\cref{lemma:reduction}]
Apply~\cref{lemma:surprise} to the sequence
\[
h_t=\frac12\left(\En_{G^{(t)}}\hat f_t-\En_{G^{(t)}}\tilde f_t\right),\qquad t\in[T].
\]
Note that both $\En_{G^{(t)}}\hat f_t$ and $\En_{G^{(t)}}\tilde f_t$ are elements of $\conv(\pm \cF)$ and are predictable.
Thus, using~\cref{lemma:surprise} gives
\begin{align*}
\En\sum_{t=1}^T \left|\En_{G^{(t)}}\hat f_t(X_t)-\En_{G^{(t)}}\tilde f_t(X_t)\right| &\lesssim \left(\frac{\log(2T)}\sigma \En\sum_{t=1}^T\sum_{s<t} \left(\En_{G^{(t)}}\hat f_t(X_s)-\En_{G^{(t)}}\tilde f_t(X_s)\right)^2 \right)^{1/2} \\
&\qquad +\sqrt{\frac{dT}\sigma}\log^2(2T/\sigma).
\end{align*}
By~\cref{lemma:empirical-prediction}, the expected double sum is in $O(d^2T\log^2(2T/\sigma))$.
Substituting this bound and using $\log(2T)\le\log(2T/\sigma)$ gives
\[
\En\sum_{t=1}^T \left|\En_{G^{(t)}}\hat f_t(X_t)-\En_{G^{(t)}}\tilde f_t(X_t)\right| \lesssim {d\log^2(2T/\sigma)\sqrt{\frac T\sigma}}.
\]
This proves~\eqref{eq:regret-reduction}.
\end{proof}

\subsection{Regret of Gaussian-perturbed exponential weights}\label{app:ew-proof}

We first prove~\cref{prop:ew-variance}, which reduces the regret bound to controlling cumulative prediction variance.
We then prove~\cref{lemma:historical-variance}, which bounds the variances on the historical covariates, and combine it with the surprise lemma to prove~\cref{lemma:ew}.

\ewvariance*

\Cref{prop:ew-variance} is the special case of the following proposition in which $\cG\subseteq\cF$ consists of binary functions and $\ell(\hat y,y)=-y\hat y$, which is linear, and hence convex and $1$-Lipschitz, in its first argument.
We use the general form in Appendix~\ref{app:real-reference-regret}.

\begin{proposition}\label{prop:ew-variance-real}
Let $\cG$ be a nonempty finite class of functions $\cX\to[-1,1]$, and let $\ell(\hat y,y)$ be convex and $1$-Lipschitz in its first argument.
For every data sequence, the predictor $\tilde f_t$ defined in~\eqref{eq:ew} satisfies the following inequality,
\[
\sum_{t=1}^T\En_{G^{(t)}} \ell(\tilde f_t(X_t),Y_t) -\min_{f\in\cG}\sum_{t=1}^T\ell(f(X_t),Y_t) \lesssim \log|\cG|+ \sum_{t=1}^T\En_{G^{(t)}} \var_{f\sim w_t(\cdot\mid G^{(t)})}(f(X_t)).
\]
\end{proposition}
\begin{proof}[Proof of~\cref{prop:ew-variance-real}]
Recall the log-partition potential from the proof of~\cref{lemma:empirical-prediction}, written for a general loss for each $t$:
\[
\widetilde\Phi_t(G^{(t)}) =\log\sum_{f\in\cG} \exp\left(-\sum_{s<t}\left[\ell(f(X_s),Y_s)+G_s^{(t)}f(X_s)\right]\right),
\]
For a fixed round $t$, we have the following identity:
\begin{align*}
\widetilde\Phi_{t+1}(G^{(t+1)}) -\widetilde\Phi_t(G_{1:t-1}^{(t+1)}) &=\log\sum_{f\in\cG} w_t(f\mid G_{1:t-1}^{(t+1)}) \exp\left(-\ell(f(X_t),Y_t)-G_t^{(t+1)}f(X_t)\right).
\end{align*}
Apply~\cref{prop:ew-variance-general} to these weights and the losses $\ell(f(X_t),Y_t)+G_t^{(t+1)}f(X_t)$ at learning rate $1$.
Their range is at most $s_t:=2(1+|G_t^{(t+1)}|)$, and
\begin{align*}
\var_{f\sim w_t(\cdot\mid G_{1:t-1}^{(t+1)})}(\ell(f(X_t),Y_t)+G_t^{(t+1)}f(X_t)) &\le (1+|G_t^{(t+1)}|)^2\var_{f\sim w_t(\cdot\mid G_{1:t-1}^{(t+1)})}(f(X_t)),
\end{align*}
since $y\mapsto\ell(y,Y_t)+G_t^{(t+1)}y$ is $(1+|G_t^{(t+1)}|)$-Lipschitz on $[-1,1]$.
Thus, cancelling $s_{t}^2$ from numerator and the denominator in~\cref{prop:ew-variance-general}, we have
\begin{align*}
&\sum_{f\in\cG}w_t(f\mid G_{1:t-1}^{(t+1)}) [\ell(f(X_t),Y_t)+G_t^{(t+1)}f(X_t)] \\
&\quad \le\widetilde\Phi_t(G_{1:t-1}^{(t+1)}) -\widetilde\Phi_{t+1}(G^{(t+1)}) \\*
&\qquad + \frac{\exp\left(2(1+|G_t^{(t+1)}|)\right)-2(1+|G_t^{(t+1)}|)-1}{4} \var_{f\sim w_t(\cdot\mid G_{1:t-1}^{(t+1)})}(f(X_t)).
\end{align*}
This inequality also holds when the loss range is zero.
Note that the last coordinate $G_t^{(t+1)}$ is centered and independent of the first $t-1$ coordinates.
In particular,
\begin{align*}
&\En_{G_t^{(t+1)}} \frac{\exp\left(2(1+|G_t^{(t+1)}|)\right)-2(1+|G_t^{(t+1)}|)-1}{4} \\*
&\quad \le\frac{\exp(2)}{4}\En_{G_t^{(t+1)}} \left(\exp\left(2G_t^{(t+1)}\right)+\exp\left(-2G_t^{(t+1)}\right)\right) =\frac{\exp(4)}{2}.
\end{align*}
Now, $G^{(t+1)}_{1:t-1}$ has the same law as $G^{(t)}$ conditionally on the history up to round $t.$ Thus, using convexity of the loss,
\begin{align}
\En_{G^{(t)}}\ell(\tilde f_t(X_t),Y_t) &\le\En_{G^{(t)}}\widetilde\Phi_t(G^{(t)}) -\En_{G^{(t+1)}}\widetilde\Phi_{t+1}(G^{(t+1)}) +\frac{\exp(4)}{2}\En_{G^{(t)}} \var_{f\sim w_t(\cdot\mid G^{(t)})}(f(X_t)) . \label{eq:round-variance-bound}
\end{align}
Note that we have $\widetilde\Phi_1=\log|\cG|$, and
\begin{align*}
\En_{G^{(T+1)}}\widetilde\Phi_{T+1}(G^{(T+1)}) &\ge\max_{f\in\cG}\En_{G^{(T+1)}} \left[-\sum_{s=1}^T\left(\ell(f(X_s),Y_s)+G_s^{(T+1)}f(X_s)\right)\right] \\
&=-\min_{f\in\cG}\sum_{s=1}^T\ell(f(X_s),Y_s),
\end{align*}
by Jensen's inequality.
Combining this with~\cref{eq:round-variance-bound} and telescoping gives the desired result.
\end{proof}

\historicalvariance*

\begin{proof}[Proof of~\cref{lemma:historical-variance}]
Write $G=G^{(t)}$ throughout this proof for notational brevity.
From~\cref{lemma:log-partition}, we have
\begin{align*}
\partial_{G_s}\tilde f_t(X_s) &=-\var_{f\sim w_t(\cdot\mid G)}(f(X_s)).
\end{align*}
Applying~\cref{lemma:stein} therefore gives:
\[
\En_G\var_{f\sim w_t(\cdot\mid G)}(f(X_s)) =-\En_G\left[G_s\tilde f_t(X_s)\right].
\]
To bound all these correlations together, we use part~(ii) of~\cref{lemma:gaussian-correlation}, which gives:
\begin{align*}
\sum_{s < t} \left(\En_G\var_{f\sim w_t(\cdot\mid G)}(f(X_s))\right)^2 &= \sum_{s < t} \left(\En_G\left[G_s\tilde f_t(X_s)\right]\right)^2 \lesssim \log 2|\cG|,
\end{align*}
thus, concluding the proof.
\end{proof}

\ewregret*

\begin{proof}[Proof of~\cref{lemma:ew}]
Apply~\cref{prop:ew-variance} to the cover from~\cref{prop:cover}.
The difference in the comparator terms over $\cG$ and $\cF$ is upper bounded as:
\begin{align*}
0 &\le\min_{g\in\cG}\sum_{t=1}^T\ell(g(X_t),Y_t) -\inf_{f\in\cF}\sum_{t=1}^T\ell(f(X_t),Y_t) \\
&\le2\sup_{f\in\cF}\inf_{g\in\cG} \sum_{t=1}^T\ind\{f(X_t)\ne g(X_t)\}.
\end{align*}
By~\cref{prop:cover}, the expectation of this difference is at most $Cd\log(2T/\sigma)+2$.
Thus, applying this estimate to~\eqref{eq:ew-variance-regret}, we have:
\begin{equation}
\En \ewreg_T \lesssim \log|\cG|+ \En \sum_{t=1}^T\En_{G^{(t)}} \var_{f\sim w_t(\cdot\mid G^{(t)})}(f(X_t)) + Cd\log(2T/\sigma)+2. \label{eq:ew-reg-via-variance}
\end{equation}
The above reduces the proof to bounding the cumulative prediction variance.
We know from~\cref{lemma:historical-variance} that variance of $f \sim w_t$ on the past covariates $X_s$ with $s < t$ can be suitably upper bounded.
To upper bound the above, we will show that variance terms can be expressed as predictable test functions of a VC class, and apply the surprise lemma~\cref{lemma:surprise}.
Indeed, for every $t$ and $x\in\cX$, we have

\[
\frac12\En_{G^{(t)}}\var_{f\sim w_t(\cdot\mid G^{(t)})}(f(x)) =\En_{G^{(t)}}\sum_{f,g\in\cG}w_t(f\mid G^{(t)})w_t(g\mid G^{(t)}) \ind\{f(x)\ne g(x)\}.
\]
This is a predictable test-function that belongs to the convex hull of the disagreement class:
\[
\left\{x\mapsto\ind\{f(x)\ne g(x)\}:f,g\in{\cF}\right\}.
\]
From~\cite{blumer1989learnability}, its VC dimension is in $O(d)$.
Thus, using~\cref{lemma:historical-variance} and applying~\cref{lemma:surprise} to the sequence of variances gives
\begin{align*}
\sum_{t=1}^T\En \var_{f\sim w_t(\cdot\mid G^{(t)})}(f(X_t)) &\lesssim \sqrt{\frac{T\log|\cG|\log(2T)}\sigma} +\sqrt{\frac{(d+1)T}\sigma}\log^2(2T/\sigma).
\end{align*}
Plugging this into~\cref{eq:ew-reg-via-variance} yields:
\[
\En \ewreg_T \lesssim \log|\cG|+ \sqrt{\frac{T\log|\cG|\log(2T)}\sigma} +\sqrt{\frac{(d+1)T}\sigma}\log^2(2T/\sigma) + d\log(2T/\sigma).
\]
To obtain~\eqref{eq:ew-rate}, we use~\cref{prop:cover} and note that $\ewreg_T\le2T$ with probability $1,$ and thus we can absorb the last term by $\sqrt{T \cdot d\log(2T/\sigma)}$.
This concludes the proof.
\end{proof}

\Cref{thm:main} follows by combining~\cref{lemma:reduction} and~\cref{lemma:ew}.

\section{Proofs for the real-valued loss case}\label{app:real-proofs}

The real-valued proof has the same three ingredients as the binary proof: a surprise lemma, a reduction to an exponential-weights reference algorithm that runs over a finite cover, and a regret bound for exponential weights.

First, we change slightly the way the reference exponential algorithm is defined.
Let $\cG$ be the cover from~\cref{prop:fat-cover} and, for each $f$, let $g_{\pi(f)}$ be the element of the cover that corresponds to $f$.
One can consider a partition of $\cF$ into cells where each cell contains all functions $f$ that share a corresponding element $g_{\pi(f)}$.
Within each cell, we consider clipping of functions towards their corresponding cover center; in particular, for a fixed valued $\eps < 1/16$, and for $f\in\cF$, define
\[
\bar f(x)=g_{\pi(f)}(x) +\max\{-2\eps,\min\{2\eps,f(x)-g_{\pi(f)}(x)\}\}.
\]
We then modify exponential weights by assigning each $g\in\cG$ a weight determined by the smallest loss attained by a clipped function in its cell.
Formally, we let
\begin{equation}
\Phi_{t,g}(G)=\sup_{\substack{f\in\cF:g_{\pi(f)}=g}} \left\{-\sum_{s<t}\left[\ell(\bar f(X_s),Y_s)+G_s\bar f(X_s)\right]\right\}, \qquad \widetilde\Phi_t(G)=\log\sum_{g\in\cG}\exp\left(\Phi_{t,g}(G)\right). \label{eq:phi-per-cell}
\end{equation}
Then, define:
\begin{equation}\label{eq:ew-real-valued}
w_t(g\mid G)=\exp\left(\Phi_{t,g}(G)-\widetilde\Phi_t(G)\right), \qquad \tilde f_t=\sum_{g\in\cG}w_t(g\mid G)\,g.
\end{equation}
Intuitively, computing the averaging weights using clipped functions as opposed to the centers $g$ accomplishes two goals.
First, by selecting the minimizer within each cell, we improve the comparison between exponential weights and FTPL.
Second, by introducing clipping, we ensure that the cell potentials are not too far from the perturbed scores of the cover centers $\cG$.
Formally, we establish the following guarantee for the clipping we use.
\begin{lemma}\label{lemma:clipped-class}
Let $\cG$ and the chosen cover elements $g_{\pi(f)}$ be as in~\cref{prop:fat-cover}, and define $\bar f$ for $f\in\cF$ as above.
\begin{enumerate}
\item[(i)] For every $f\in\cF$ and $x\in\cX$, we have $\bar f(x)\in[-1,1]$, $|\bar f(x)-g_{\pi(f)}(x)|\le2\eps$, and
\[
|\bar f(x)-f(x)|\le2\ind\{|f(x)-g_{\pi(f)}(x)|>2\eps\}.
\]
\item[(ii)] With probability at least $1-1/T$,
\[
\sup_{f\in\cF}\sum_{s=1}^T|f(X_s)-\bar f(X_s)| \lesssim\fat_{c\eps}(\cF)\log^2\frac{CT}{\sigma\eps}.
\]
\end{enumerate}
\end{lemma}
\begin{proof}
For part~(i), $\bar f-g_{\pi(f)}$ is $f-g_{\pi(f)}$ clipped to $[-2\eps,2\eps]$.
Hence $\bar f(x)$ lies between $g_{\pi(f)}(x)$ and $f(x)$, so it belongs to $[-1,1]$ and satisfies $|\bar f(x)-f(x)|\le|f(x)-g_{\pi(f)}(x)|\le2$.
It equals $f(x)$ whenever $|f(x)-g_{\pi(f)}(x)|\le2\eps$.
For part~(ii), part~(i) gives $|f(X_s)-\bar f(X_s)|\le2\ind\{|f(X_s)-g_{\pi(f)}(X_s)|>2\eps\}$.
Summing over $s\le T$ and applying~\eqref{eq:real-cover-count} proves the claim.
\end{proof}

\subsection{Finite-scale surprise bounds}\label{app:real-valued}

We begin by adapting our surprise bound argument for real-valued classes.

\begin{lemma}\label{lemma:real-transfer}
Let $0<\eps\le1/16$, and let $\{h_t\}_{t\in[T]}$ be a predictable sequence of functions, each of which is a convex combination or a probability mixture of elements of $\pm\cF$.
Then
\begin{align*}
\En\sum_{t=1}^T|h_t(X_t)| &\lesssim T\eps\sqrt{\frac{\log(2T)}{\sigma}}+\sqrt{\frac{\log(2T)}{\sigma} \En\sum_{t=1}^T\sum_{s<t}h_t(X_s)^2} +\sqrt{\frac{\left(1+\fat_{c\eps}(\cF)\right)T}{\sigma}} \log^{5/2}\frac{CT}{\sigma\eps}.
\end{align*}
\end{lemma}
Compared with~\cref{lemma:surprise}, the VC dimension is replaced by the fat-shattering dimension at scale $c\eps$, at the price of the additive term $T\eps\sqrt{\log(2T)/\sigma}$.
This term pays for replacing each function by its chosen cover element, which is within $2\eps$ of it except at a few covariates.
\begin{proof}[Proof of~\cref{lemma:real-transfer}]
Throughout this proof, suprema over $h$ range over $\conv(\pm \cF)$.
We first establish the following analogue of~\cref{lemma:mixture-norm}:
\begin{align}
\En\sup_h \left\{\sum_{s=1}^T\En_{\bar X_s}h(\bar X_s)^2-12\sum_{s=1}^T h(X_s)^2\right\} &\lesssim T\eps^2+\fat_{c\eps}(\cF)\log^4\frac{CT}{\sigma\eps}. \label{eq:real-norm}
\end{align}
As in the proof of~\cref{lemma:mixture-norm}, we use the layer-cake identity and Maurey's empirical method (\cref{prop:maurey}) to reduce the claim to comparison for a binary function class, after which we apply~\citet[Theorem~10]{block2024performance}.
To start, using the layer-cake identity, we have
\[
\sum_{s=1}^T\En_{\bar X_s}h(\bar X_s)^2 \le4T\eps^2+\int_{2\eps}^1 2\lambda\sum_{s=1}^T \P[|h(\bar X_s)|\ge\lambda]\dd\lambda.
\]
Also, note that we have:
\[
\sum_{s=1}^T h(X_s)^2\ge\int_{2\eps}^1\frac{\lambda}{2}\sum_{s=1}^T \ind\{|h(X_s)|\ge\lambda/2\}\dd\lambda.
\]
Thus,
\begin{align}
&\sum_{s=1}^T\En_{\bar X_s}h(\bar X_s)^2 -12\sum_{s=1}^T h(X_s)^2 \notag \\
&\quad \le4T\eps^2+\int_{2\eps}^1 2\lambda \left\{\sum_{s=1}^T\P[|h(\bar X_s)|\ge\lambda] -3\sum_{s=1}^T\ind\{|h(X_s)|\ge\lambda/2\}\right\}\dd\lambda. \label{eq:real-layer-cake}
\end{align}
For $2\eps\le\lambda\le1$, apply~\cref{prop:fat-cover} at scale $\lambda/16$ and horizon $T$ to obtain a fixed finite class $\cG_\lambda\subseteq\cF$ and a choice of cover element $g_{\pi(f)}$ for every $f\in\cF$.
Then
\[
\log|\cG_\lambda|\lesssim\fat_{c\lambda}(\cF)\log^2\frac{CT}{\sigma\lambda}.
\]
The tangent sequence is also $\sigma$-smooth with respect to $\mu$.
Thus, applying~\cref{prop:fat-cover} to both sequences and taking a union bound, with probability at least $1-2/T$ we have
\begin{equation}\label{eq:good-approx-event}
\sup_{f\in\cF}\sum_{s=1}^T \left(\ind\{|f(X_s)-g_{\pi(f)}(X_s)|>\lambda/8\} +\ind\{|f(\bar X_s)-g_{\pi(f)}(\bar X_s)|>\lambda/8\}\right) \lesssim\fat_{c\lambda}(\cF)\log^2\frac{CT}{\sigma\lambda}.
\end{equation}
Take $K=\lceil C\lambda^{-2}\log(4T)\rceil$ and define the fixed binary class
\[
\cF_\lambda= \left\{ x\mapsto\ind\left\{\left|\frac1K\sum_{i=1}^K f_i(x)\right| \ge\frac{3\lambda}{4}\right\}: f_1,\ldots,f_K\in\pm\cG_\lambda\right\}.
\]
Applying~\cref{prop:maurey} to the $2T$ points $X_{1:T},\bar X_{1:T}$, for $C$ large enough and every $h\in\conv(\pm\cF)$ there exist $f_1,\ldots,f_K\in\pm\cF$ such that
\[
\max_{x\in \{X_{s},\bar X_{s}, s \in [T]\}} \left|h(x)-\frac1K\sum_{i=1}^K f_i(x)\right|\le\lambda/8.
\]
For each $f_i$, consider the corresponding element $g_i := g_{\pi(f_i)}$ of the cover.
On the event in~\cref{eq:good-approx-event}, we have
\begin{align*}
&\sum_{s=1}^T\left( \ind\left\{\max_{i\le K}|f_i(X_s)-g_i(X_s)|>\lambda/8\right\} +\ind\left\{\max_{i\le K}|f_i(\bar X_s)-g_i(\bar X_s)|>\lambda/8\right\}\right) \\
&\quad \lesssim K\fat_{c\lambda}(\cF)\log^2\frac{CT}{\sigma\lambda}.
\end{align*}
Call $x \in \{X_s,\bar X_s, s \in[T]\}$ good if $|f_i(x) - g_i(x)| \le \lambda/8$, and call it bad otherwise.
Let $\bar g=K^{-1}\sum_{i=1}^K g_i$.
For any good $x$, the Maurey approximation and the cover guarantee imply
\[
|h(x)-\bar g(x)| \le\left|h(x)-\frac1K\sum_{i=1}^K f_i(x)\right| +\frac1K\sum_{i=1}^K|f_i(x)-g_i(x)| \le\frac{\lambda}{8}+\frac{\lambda}{8} =\frac{\lambda}{4}.
\]
Consequently, $g=\ind\{|\bar g|\ge3\lambda/4\}\in\cF_\lambda$ satisfies,
\[
\ind\{|h(x)|\ge\lambda\}\le g(x) \le\ind\{|h(x)|\ge\lambda/2\}
\]
at all but $CK\fat_{c\lambda}(\cF)\log^2(CT/(\sigma\lambda))$ points among the two sequences.
These inequalities allow us to replace the two threshold indicators by $g$, on all but $CK\fat_{c\lambda}(\cF)\log^2(CT/(\sigma\lambda))$ bad points.
This gives:
\begin{align}
&\En\sup_h\left\{\sum_{s=1}^T\P[|h(\bar X_s)|\ge\lambda] -3\sum_{s=1}^T\ind\{|h(X_s)|\ge\lambda/2\}\right\} \notag \\
&\quad \le\En\sup_h\left\{\sum_{s=1}^T\ind\{|h(\bar X_s)|\ge\lambda\} -3\sum_{s=1}^T\ind\{|h(X_s)|\ge\lambda/2\}\right\} \notag \\
&\quad \lesssim K\fat_{c\lambda}(\cF)\log^2\frac{CT}{\sigma\lambda} +\En\sup_{g\in\cF_\lambda}\sum_{s=1}^T\left(g(\bar X_s)-3g(X_s)\right) \notag \\
&\quad \lesssim K\fat_{c\lambda}(\cF)\log^2\frac{CT}{\sigma\lambda} +\log|\cF_\lambda|. \label{eq:real-growth-fn-bound}
\end{align}
The last inequality applies~\citet[Theorem~10]{block2024performance} to the fixed binary class $\cF_\lambda$, whose Wills functional is bounded by its cardinality.

A counting argument gives $|\cF_\lambda|\le(2|\cG_\lambda|)^K$.
Using the bound on $|\cG_\lambda|$ and $\fat_{c\lambda}(\cF)\le\fat_{c\eps}(\cF)$ for $\lambda\ge2\eps$, the last line of~\eqref{eq:real-growth-fn-bound} is at most
\[
CK\fat_{c\lambda}(\cF)\log^2\frac{CT}{\sigma\lambda} \lesssim\fat_{c\eps}(\cF)\lambda^{-2}\log^3\frac{CT}{\sigma\lambda}.
\]
Combining~\cref{eq:real-growth-fn-bound,eq:real-layer-cake}, taking the supremum inside the integral, and then taking expectation gives
\begin{align*}
\En\sup_h \left\{\sum_{s=1}^T\En_{\bar X_s}h(\bar X_s)^2 -12\sum_{s=1}^T h(X_s)^2\right\} &\lesssim T\eps^2+\int_{2\eps}^1 2\lambda\fat_{c\eps}(\cF)\lambda^{-2}\log^3\frac{CT}{\sigma\lambda}\dd\lambda \\
&\lesssim T\eps^2+\fat_{c\eps}(\cF)\log^4\frac{CT}{\sigma\eps},
\end{align*}
where the last step uses $\int_{2\eps}^1\lambda^{-1}\dd\lambda\le\log(1/\eps)$.
This proves~\eqref{eq:real-norm}.

If $\sigma T<1$, the claim follows from $|h_t|\le1$, so assume $\sigma T\ge1$.
Apply~\cref{lemma:density-transfer} to $|h_t|$ and~\eqref{eq:real-norm} on each prefix $X_{1:t-1}$.
Summing the prefix approximation terms gives $\sum_{t=1}^T(t-1)\eps^2\le T^2\eps^2$, so
\begin{align*}
\En\sum_{t=1}^T|h_t(X_t)| &\lesssim\frac{\log(2T)}\sigma +\left[\frac{\log(2T)}\sigma \left(\En\sum_{t=1}^T\sum_{s<t}h_t(X_s)^2 +T^2\eps^2+\fat_{c\eps}(\cF)T\log^4\frac{CT}{\sigma\eps}\right)\right]^{1/2} \\
&\lesssim T\eps\sqrt{\frac{\log(2T)}\sigma} +\sqrt{\frac{\log(2T)}\sigma \En\sum_{t=1}^T\sum_{s<t}h_t(X_s)^2} +\sqrt{\frac{\fat_{c\eps}(\cF)T}\sigma}\log^{5/2}\frac{CT}{\sigma\eps} +\frac{\log(2T)}\sigma,
\end{align*}
where the last step uses $\sqrt{a+b}\le\sqrt a+\sqrt b$ and $\log(2T)\le\log(CT/(\sigma\eps))$.
Since $\sigma T\ge1$, we have $1/\sigma\le\sqrt{T/\sigma}$, so the final term is absorbed into the preceding one.

\end{proof}

\subsection{Proof of the prediction comparison on the past}\label{app:real-prediction}

We begin by proving that predictions of FTPL~\cref{eq:ftpl} and the reference from ~\cref{eq:ew-real-valued} in round $t$ are close on past covariates $X_{s}$ with $s < t$.

\begin{lemma}\label{lemma:real-prediction-comparison}
For every $t\le T$,
\begin{equation*}
\En\sum_{s<t} \left(\En_{G^{(t)}}\hat f_t(X_s)-\En_{G^{(t)}}\tilde f_t(X_s)\right)^2 \lesssim\fat_{c\eps}(\cF)^2 \log^5\frac{CT}{\sigma\eps}+T\eps^2+1.
\end{equation*}
\end{lemma}

\begin{proof}[Proof of~\cref{lemma:real-prediction-comparison}]
Fix the history $(X_{1:t-1},Y_{1:t-1})$ and let $G:=G^{(t)}$ to lighten the notational load.
The case $t=1$ is immediate, so suppose that $t\ge2$.
Recall the definition of potentials from~\cref{eq:phi-per-cell}:
\begin{align*}
\Phi_t(G) &=\sup_{f\in\cF}\left\{-\sum_{s<t}\left[\ell(f(X_s),Y_s)+G_sf(X_s)\right]\right\}, \\
\widetilde\Phi_t(G) &=\log\sum_{g\in\cG}\exp\left(\Phi_{t,g}(G)\right) \\
&=\max_{w\in\Delta(\cG)} \left[\En_{g\sim w}\Phi_{t,g}(G)+H(w)\right],
\end{align*}
with $\Phi_{t,g}$ as in~\cref{eq:phi-per-cell}.
Note that $w_t(\cdot\mid G)$ from~\cref{eq:ew-real-valued} maximize the latter potential in its variational form.
Note that, by absolute continuity of the Gaussians, the set of all functions that maximize the former potential share predictions on the past covariates $X_s$.
Differentiating the two potentials therefore gives, for every $s<t$ and almost every $G$, by Danskin's theorem~\citep[Appendix~B]{bertsekas2016nonlinear} and~\cref{lemma:log-partition}:
\[
\partial_{G_s}\Phi_t(G)=-\hat f_t(X_s),\qquad \partial_{G_s}\widetilde\Phi_t(G) =\sum_{g\in\cG}w_t(g\mid G)\partial_{G_s}\Phi_{t,g}(G).
\]
Applying Danskin's theorem to $\Phi_{t,g}$, and using the definition of clipped functions $\bar f$, we have:
\begin{equation}\label{eq:clipped-cell-gradient}
\left|\partial_{G_s}\Phi_{t,g}(G)+g(X_s)\right|\le2\eps,
\end{equation}
and thus,
\begin{equation}\label{eq:clipped-gradient}
\left|\partial_{G_s}\widetilde\Phi_t(G) +\tilde f_t(X_s)\right|\le2\eps.
\end{equation}
Both potentials are Lipschitz.
Applying Stein's lemma (\cref{lemma:stein}) and then~\cref{lemma:gaussian-correlation}(i) gives:
\[
\left\|\En_G\nabla_G(\Phi_t(G)-\widetilde\Phi_t(G))\right\|_2^2 \le\En_G(\Phi_t-\widetilde\Phi_t)^2.
\]
Combining this with~\eqref{eq:clipped-gradient} yields
\begin{equation}
\sum_{s<t} \left(\En_G\hat f_t(X_s)-\En_G\tilde f_t(X_s)\right)^2 \le2\En_G(\Phi_t-\widetilde\Phi_t)^2+8(t-1)\eps^2. \label{eq:real-pot-diff-ub}
\end{equation}
It remains to bound the difference of the potentials.
Since the entropy of any distribution on $\cG$ is at most $\log |\cG|$, the variational form of $\widetilde\Phi_t$ gives
\[
\max_{g\in\cG}\Phi_{t,g}(G)\le\widetilde\Phi_t(G)\le\max_{g\in\cG}\Phi_{t,g}(G)+\log|\cG|.
\]
The cells $\{f\in\cF:g_{\pi(f)}=g\}$ partition $\cF$, so $\max_{g\in\cG}\Phi_{t,g}(G)$ is the supremum defining $\Phi_t(G)$ with each $f$ replaced by $\bar f$.
For each $s<t$, the map $y\mapsto\ell(y,Y_s)+G_sy$ is $(1+|G_s|)$-Lipschitz on $[-1,1]$.
Hence, on the event of~\cref{lemma:clipped-class}(ii), which has probability at least $1-1/T$ over the history, using~\eqref{eq:real-cover-size} to bound $\log|\cG|$,
\begin{align*}
|\Phi_t(G)-\widetilde\Phi_t(G)| &\le\log|\cG|+\left(1+\max_{s<t}|G_s|\right) \sup_{f\in\cF}\sum_{s<t}|f(X_s)-\bar f(X_s)| \\
&\lesssim\left(1+\max_{s<t}|G_s|\right) \fat_{c\eps}(\cF)\log^2\frac{CT}{\sigma\eps}.
\end{align*}
Since $\En_G(1+\max_{s<t}|G_s|)^2\lesssim\log(2T)$, it follows that
\[
\En_G(\Phi_t-\widetilde\Phi_t)^2 \lesssim\fat_{c\eps}(\cF)^2\log^5\frac{CT}{\sigma\eps}.
\]
Plugging this into~\cref{eq:real-pot-diff-ub} gives the claimed bound on the above event.
The complement of this high probability event contributes at most $O(1)$.
This concludes the proof.
\end{proof}

\subsection{Proof of the reduction to exponential weights}\label{app:real-reduction}

We now apply the surprise lemma~\cref{lemma:real-transfer} to the difference of predictions.

\begin{lemma}\label{lemma:real-reduction}
For $\tilde f_t$ as defined in ~\cref{eq:ew-real-valued}, we have the following bound:
\begin{equation}
\En\sum_{t=1}^T \left|\En_{G^{(t)}}\hat f_t(X_t)-\En_{G^{(t)}}\tilde f_t(X_t)\right| \lesssim\left[\fat_{c\eps}(\cF)\sqrt{\frac T\sigma} +\frac{T\eps}{\sqrt\sigma}\right] \log^3\frac{CT}{\sigma\eps}+\frac{\log(2T)}\sigma. \label{eq:real-prediction-reduction}
\end{equation}
\end{lemma}

\begin{proof}[Proof of~\cref{lemma:real-reduction}]
Apply~\cref{lemma:real-transfer} to the sequence
\[
h_t=\frac12\left(\En_{G^{(t)}}\hat f_t-\En_{G^{(t)}}\tilde f_t\right),\qquad t\in[T].
\]
Note that both $\En_{G^{(t)}}\hat f_t$ and $\En_{G^{(t)}}\tilde f_t$ are probability mixtures of elements of $\cF$ and are predictable.
Thus, $h_t$ is a probability mixture of elements of $\pm\cF$, and using~\cref{lemma:real-transfer} gives
\begin{align*}
&\En\sum_{t=1}^T \left|\En_{G^{(t)}}\hat f_t(X_t)-\En_{G^{(t)}}\tilde f_t(X_t)\right| \\
&\quad \lesssim T\eps\sqrt{\frac{\log(2T)}\sigma}+ \left(\frac{\log(2T)}\sigma \En\sum_{t=1}^T\sum_{s<t} \left(\En_{G^{(t)}}\hat f_t(X_s)-\En_{G^{(t)}}\tilde f_t(X_s)\right)^2 \right)^{1/2} \\
&\qquad +\sqrt{\frac{\fat_{c\eps}(\cF)T}\sigma} \log^{5/2}\frac{CT}{\sigma\eps}.
\end{align*}
By~\cref{lemma:real-prediction-comparison}, the expected double sum is in $O(T\fat_{c\eps}(\cF)^2\log^5(CT/(\sigma\eps))+T^2\eps^2+T)$.
Substituting this bound and using $\log(2T)\le\log(CT/(\sigma\eps))$ gives
\begin{align*}
\En\sum_{t=1}^T \left|\En_{G^{(t)}}\hat f_t(X_t)-\En_{G^{(t)}}\tilde f_t(X_t)\right| &\lesssim\left[\fat_{c\eps}(\cF)\sqrt{\frac T\sigma} +\frac{T\eps}{\sqrt\sigma}\right] \log^3\frac{CT}{\sigma\eps}+\frac{\log(2T)}\sigma.
\end{align*}
This proves~\eqref{eq:real-prediction-reduction}.
\end{proof}

\subsection{Regret of Gaussian-perturbed exponential weights}\label{app:real-reference-regret}

We first prove~\cref{prop:clipped-reference-variance}, which reduces the regret bound to controlling cumulative prediction variance.
We then prove~\cref{lemma:clipped-variance}, which bounds the variances on the historical covariates, and combine it with the surprise lemma to prove~\cref{lemma:clipped-reference-regret}.

\begin{proposition}\label{prop:clipped-reference-variance}
For every data sequence, the predictor $\tilde f_t$ defined in ~\cref{eq:ew-real-valued} satisfies the following inequality,
\begin{align}
&\sum_{t=1}^T\ell(\En_{G^{(t)}}\tilde f_t(X_t),Y_t) -\inf_{f\in\cF}\sum_{t=1}^T\ell(\bar f(X_t),Y_t) \notag \\*
&\quad \lesssim\log|\cG|+T\eps+ \sum_{t=1}^T\En_{G^{(t)}} \var_{g\sim w_t(\cdot\mid G^{(t)})}(g(X_t)). \label{eq:clipped-fixed-data-regret}
\end{align}
\end{proposition}

\begin{proof}[Proof of~\cref{prop:clipped-reference-variance}]
We follow the proof of~\cref{prop:ew-variance-real}, with the log-partition potential from~\cref{eq:phi-per-cell}
\[
\widetilde\Phi_t(G^{(t)}) =\log\sum_{g\in\cG}\exp\left(\Phi_{t,g}(G^{(t)})\right), \qquad 1\le t\le T+1.
\]
For a fixed round $t$, we have the following identity:
\begin{align*}
\widetilde\Phi_{t+1}(G^{(t+1)}) -\widetilde\Phi_t(G_{1:t-1}^{(t+1)}) &=\log\sum_{g\in\cG}w_t(g\mid G_{1:t-1}^{(t+1)}) \exp\left(\Phi_{t+1,g}(G^{(t+1)}) -\Phi_{t,g}(G_{1:t-1}^{(t+1)})\right).
\end{align*}
By~\cref{lemma:clipped-class}(i), the potential increments satisfy
\[
\left|\Phi_{t,g}(G_{1:t-1}^{(t+1)})-\Phi_{t+1,g}(G^{(t+1)}) -\ell(g(X_t),Y_t)-G_t^{(t+1)}g(X_t)\right| \le2\eps(1+|G_t^{(t+1)}|).
\]
Consequently,
\begin{align*}
\widetilde\Phi_{t+1}(G^{(t+1)}) -\widetilde\Phi_t(G_{1:t-1}^{(t+1)}) &\le\log\sum_{g\in\cG}w_t(g\mid G_{1:t-1}^{(t+1)}) \exp\left(-\ell(g(X_t),Y_t)-G_t^{(t+1)}g(X_t)\right) \\*
&\qquad +2\eps(1+|G_t^{(t+1)}|).
\end{align*}
Apply~\cref{prop:ew-variance-general} to these weights and the losses $\ell(g(X_t),Y_t)+G_t^{(t+1)}g(X_t)$ at learning rate $1$.
Their range is at most $s_t:=2(1+|G_t^{(t+1)}|)$, and
\begin{align*}
\var_{g\sim w_t(\cdot\mid G_{1:t-1}^{(t+1)})} (\ell(g(X_t),Y_t)+G_t^{(t+1)}g(X_t)) &\le(1+|G_t^{(t+1)}|)^2 \var_{g\sim w_t(\cdot\mid G_{1:t-1}^{(t+1)})}(g(X_t)),
\end{align*}
since $u\mapsto\ell(u,Y_t)+G_t^{(t+1)}u$ is $(1+|G_t^{(t+1)}|)$-Lipschitz.
Thus, cancelling $s_{t}^2$ from numerator and the denominator in~\cref{prop:ew-variance-general}, we have
\begin{align*}
&\sum_{g\in\cG}w_t(g\mid G_{1:t-1}^{(t+1)}) [\ell(g(X_t),Y_t)+G_t^{(t+1)}g(X_t)] \\
&\quad \le\widetilde\Phi_t(G_{1:t-1}^{(t+1)}) -\widetilde\Phi_{t+1}(G^{(t+1)})+2\eps(1+|G_t^{(t+1)}|) \\
&\qquad + \frac{\exp\left(2(1+|G_t^{(t+1)}|)\right)-2(1+|G_t^{(t+1)}|)-1}{4} \var_{g\sim w_t(\cdot\mid G_{1:t-1}^{(t+1)})}(g(X_t)).
\end{align*}
Note that the last coordinate $G_t^{(t+1)}$ is centered and independent of the first $t-1$ coordinates.
In particular,
\begin{align*}
&\En_{G_t^{(t+1)}} \frac{\exp\left(2(1+|G_t^{(t+1)}|)\right)-2(1+|G_t^{(t+1)}|)-1}{4} \\*
&\quad \le\frac{\exp(2)}{4}\En_{G_t^{(t+1)}} \left(\exp\left(2G_t^{(t+1)}\right)+\exp\left(-2G_t^{(t+1)}\right)\right) =\frac{\exp(4)}{2}.
\end{align*}
Now, $G^{(t+1)}_{1:t-1}$ has the same law as $G^{(t)}$ conditionally on the history up to round $t$.
Thus, using convexity of the loss and $\En|G_t^{(t+1)}|\le1$,
\begin{align}
\ell(\En_{G^{(t)}}\tilde f_t(X_t),Y_t) &\le\En_{G^{(t)}}\widetilde\Phi_t(G^{(t)}) -\En_{G^{(t+1)}}\widetilde\Phi_{t+1}(G^{(t+1)}) \notag \\*
&\qquad +4\eps+\frac{\exp(4)}{2}\En_{G^{(t)}} \var_{g\sim w_t(\cdot\mid G^{(t)})}(g(X_t)). \label{eq:clipped-round-variance-bound}
\end{align}
Note that we have $\widetilde\Phi_1=\log|\cG|$, and
\begin{align*}
\En_{G^{(T+1)}}\widetilde\Phi_{T+1}(G^{(T+1)}) &\ge\sup_{f\in\cF}\En_{G^{(T+1)}}\left[-\sum_{s=1}^T\left(\ell(\bar f(X_s),Y_s)+G_s^{(T+1)}\bar f(X_s)\right)\right] \\*
&=-\inf_{f\in\cF}\sum_{s=1}^T\ell(\bar f(X_s),Y_s),
\end{align*}
by Jensen's inequality.
Combining this with~\cref{eq:clipped-round-variance-bound} and telescoping gives the desired result.
\end{proof}

\begin{lemma}\label{lemma:clipped-variance}
For every $t\le T$ and every fixed history $(X_s,Y_s)_{s<t}$, the weights from~\cref{eq:ew-real-valued} satisfy
\begin{equation*}
\sum_{s<t}\left(\En_{G^{(t)}} \var_{g\sim w_t(\cdot\mid G^{(t)})}(g(X_s))\right)^2 \lesssim\log(2|\cG|)+T\eps^2.
\end{equation*}
\end{lemma}

\begin{proof}[Proof of~\cref{lemma:clipped-variance}]
Write $G=G^{(t)}$ throughout this proof for notational brevity.
From~\cref{lemma:log-partition}, we have
\begin{align*}
\partial_{G_s}\tilde f_t(X_s) &= \cov_{g\sim w_t(\cdot\mid G)} (g(X_s),\partial_{G_s}\Phi_{t,g}(G)),
\end{align*}
Using~\cref{eq:clipped-cell-gradient}, we have
\[
\left|\partial_{G_s}\tilde f_t(X_s) +\var_{g\sim w_t(\cdot\mid G)}(g(X_s))\right|\le2\eps.
\]
After applying~\cref{lemma:stein}, we thus get:
\[
\En_G\var_{g\sim w_t(\cdot\mid G)}(g(X_s)) \le-\En_G\left[G_s\tilde f_t(X_s)\right]+2\eps.
\]
To bound the sum of correlations, we invoke part~(ii) of~\cref{lemma:gaussian-correlation}.
This gives
\begin{align*}
\sum_{s < t} \left(\En_G\var_{g\sim w_t(\cdot\mid G)}(g(X_s))\right)^2 &\lesssim \sum_{s < t} \left(\En_G\left[G_s\tilde f_t(X_s)\right]\right)^2 + \eps^2 t \\
&\lesssim \log (2 |\cG|) + \eps^2 t,
\end{align*}
which concludes the proof.
\end{proof}

Finally, combining the preceding bounds, we obtain the following result.

\begin{lemma}\label{lemma:clipped-reference-regret}
The regret of the modified exponential weights algorithm from~\cref{eq:ew-real-valued} satisfies
\begin{align*}
&\En\left[\sum_{t=1}^T\ell(\En_{G^{(t)}}\tilde f_t(X_t),Y_t) -\inf_{f\in\cF}\sum_{t=1}^T\ell(f(X_t),Y_t)\right] \\
&\quad \lesssim\fat_{c\eps}(\cF)\log^2\frac{CT}{\sigma\eps}+T\eps +\sqrt{\frac T\sigma} \left[\sqrt{\log(2|\cG|)}\log^{2}\frac{2T}{\sigma} +\sqrt{T\log(2T)}\eps\right] +\frac{\log(2T)}\sigma.
\end{align*}
\end{lemma}

\begin{proof}[Proof of~\cref{lemma:clipped-reference-regret}]
Apply~\cref{prop:clipped-reference-variance} to the cover from~\cref{prop:fat-cover}.
The difference in the comparator terms over the clipped functions and $\cF$ is upper bounded as:
\begin{align*}
\left|\inf_{f\in\cF}\sum_{t=1}^T\ell(\bar f(X_t),Y_t) -\inf_{f\in\cF}\sum_{t=1}^T\ell(f(X_t),Y_t)\right| &\le\sup_{f\in\cF}\sum_{t=1}^T|\bar f(X_t)-f(X_t)|,
\end{align*}
since $\ell(\cdot,Y_t)$ is $1$-Lipschitz.
The right-hand side is at most $2T$, so, by~\cref{lemma:clipped-class}(ii), its expectation is at most $C\fat_{c\eps}(\cF)\log^2(CT/(\sigma\eps))+ O(1)$.
Thus, applying this estimate to~\eqref{eq:clipped-fixed-data-regret}, we have:
\begin{align}
&\En\left[\sum_{t=1}^T\ell(\En_{G^{(t)}}\tilde f_t(X_t),Y_t) -\inf_{f\in\cF}\sum_{t=1}^T\ell(f(X_t),Y_t)\right] \notag \\*
&\quad \lesssim\log|\cG|+T\eps+ \En\sum_{t=1}^T\En_{G^{(t)}} \var_{g\sim w_t(\cdot\mid G^{(t)})}(g(X_t)) +\fat_{c\eps}(\cF)\log^2\frac{CT}{\sigma\eps}. \label{eq:clipped-reference-regret}
\end{align}
The above reduces the proof to bounding the cumulative prediction variance.
We know from~\cref{lemma:clipped-variance} that the variance of $g(X_s)$ under $g\sim w_t$ on the past covariates $X_s$, $s<t$, can be suitably upper bounded.
To upper bound the above, we will show that variance terms can be expressed as predictable test functions in the convex hull of a finite class, and apply~{\cref{cor:finite-surprise}}.
Indeed, for every $t$ and $x\in\cX$, we have
\[
\frac12\En_{G^{(t)}} \var_{g\sim w_t(\cdot\mid G^{(t)})}(g(x)) =\En_{G^{(t)}}\sum_{g,g'\in\cG} w_t(g\mid G^{(t)})w_t(g'\mid G^{(t)}) \frac{(g(x)-g'(x))^2}{4}.
\]
This is a predictable test-function that belongs to the convex hull of the squared-difference class:
\[
\left\{x\mapsto(g(x)-g'(x))^2/4: g,g'\in\cG\right\}.
\]
It has at most $|\cG|^2$ elements.
If $|\cG|=1$, all variances vanish; otherwise, using~\cref{lemma:clipped-variance} and applying~\cref{cor:finite-surprise} to the sequence of variances, with the finite class from above in place of $\cF$, gives
\begin{align*}
&\En\sum_{t=1}^T\En_{G^{(t)}} \var_{g\sim w_t(\cdot\mid G^{(t)})}(g(X_t)) \\
&\quad \lesssim\sqrt{\frac{T\log(2T)}\sigma \left(\log(2|\cG|)+T\eps^2\right)} +\sqrt{\frac{T\log|\cG|}\sigma}\log^{2}\frac{2T}{\sigma} +\frac{\log(2T)}\sigma.
\end{align*}
Plugging this into~\cref{eq:clipped-reference-regret} yields:
\begin{align*}
&\En\left[\sum_{t=1}^T\ell(\En_{G^{(t)}}\tilde f_t(X_t),Y_t) -\inf_{f\in\cF}\sum_{t=1}^T\ell(f(X_t),Y_t)\right] \\
&\quad \lesssim\log|\cG|+T\eps +\sqrt{\frac{T\log(2T)}\sigma \left(\log(2|\cG|)+T\eps^2\right)} \\
&\qquad +\sqrt{\frac{T\log|\cG|}\sigma}\log^{2}\frac{2T}{\sigma} +\frac{\log(2T)}\sigma +\fat_{c\eps}(\cF)\log^2\frac{CT}{\sigma\eps}.
\end{align*}
To obtain the stated regret bound, we use~\cref{prop:fat-cover} to bound $\log|\cG|$ and use $\sqrt{a+b}\le\sqrt a+\sqrt b$.
This concludes the proof.
\end{proof}

\subsection{Proof of the real-valued theorem}\label{app:convex-proofs}

The preceding bounds control the Gaussian mean of FTPL.
It remains to account for the finite average in~\eqref{eq:algo-rv} and combine the prediction and reference-regret bounds.

\mainrealvalued*
\begin{proof}[Proof of~\cref{thm:main-real-valued}]
Fix $0<\eps\le1/16$ with $\fat_{c\eps}(\cF)<\infty$.
If $1+\fat_{c\eps}(\cF)>T$ or $T<1/\sigma$, then $2T$ is at most a constant times the expression inside the infimum in~\eqref{eq:main-rate-real-valued}, so the bound at this $\eps$ follows from $\reg_T\le2T$.
Otherwise, $1+\fat_{c\eps}(\cF)\le T$ and $T\ge1/\sigma$, and we construct the fixed cover and the modified exponential weights algorithm from~\cref{eq:ew-real-valued}.

Conditional on $\cH_{t-1},X_t,Y_t$, the average of the $t$ independent oracle predictions has variance at most $1/t$ about $\En_{G^{(t)}}\hat f_t(X_t)$.
Cauchy--Schwarz and Lipschitzness give
\[
\En\left[\ell(\hat Y_t,Y_t) -\ell(\En_{G^{(t)}}\hat f_t(X_t),Y_t) \mid\cH_{t-1},X_t,Y_t\right]\le t^{-1/2}.
\]
Summing over $t$ and using $\sum_{t\le T}t^{-1/2}\le2\sqrt T$, we obtain
\begin{equation}\label{eq:real-averaging-cost}
\En\reg_T\le2\sqrt T+\En\left[\sum_{t=1}^T\ell(\En_{G^{(t)}}\hat f_t(X_t),Y_t)-\inf_{f\in\cF}\sum_{t=1}^T\ell(f(X_t),Y_t)\right].
\end{equation}
We decompose the regret of the Gaussian mean on the right-hand side as
\begin{align*}
\sum_{t=1}^T\ell(\En_{G^{(t)}}\hat f_t(X_t),Y_t)-\inf_{f\in\cF}\sum_{t=1}^T\ell(f(X_t),Y_t) &=\sum_{t=1}^T\left[\ell(\En_{G^{(t)}}\hat f_t(X_t),Y_t) -\ell(\En_{G^{(t)}}\tilde f_t(X_t),Y_t)\right] \\
&\qquad +\left[\sum_{t=1}^T\ell(\En_{G^{(t)}}\tilde f_t(X_t),Y_t) -\inf_{f\in\cF}\sum_{t=1}^T\ell(f(X_t),Y_t)\right].
\end{align*}
Since the loss is $1$-Lipschitz, the first sum is bounded by~\cref{lemma:real-reduction} in expectation.
\Cref{lemma:clipped-reference-regret} bounds the expectation of the second bracket.
Plugging these bounds into~\eqref{eq:real-averaging-cost}, and using~\eqref{eq:real-cover-size} to bound $\sqrt{\log(2|\cG|)}\lesssim\sqrt{1+\fat_{c\eps}(\cF)}\log(CT/(\sigma\eps))$ and $\log(2T/\sigma)\le\log(CT/(\sigma\eps))$, yields
\begin{align*}
\En\reg_T &\lesssim\sqrt T+\left[\left(1+\fat_{c\eps}(\cF)\right)\sqrt{\frac T\sigma} +\frac{T\eps}{\sqrt\sigma}\right] \log^3\frac{CT}{\sigma\eps} +\frac{\log(2T)}\sigma \\
&\qquad +\fat_{c\eps}(\cF)\log^2\frac{CT}{\sigma\eps}+T\eps \\
&\qquad +\sqrt{\frac T\sigma} \left[\sqrt{1+\fat_{c\eps}(\cF)}\log^{3}\frac{CT}{\sigma\eps}+\sqrt{T\log(2T)}\eps\right].
\end{align*}
To obtain~\eqref{eq:main-rate-real-valued}, note that each term on the right-hand side is at most a constant times
\[
\left(1+\fat_{c\eps}(\cF)+\sqrt T\eps\right) \sqrt{\frac T\sigma}\log^{3}\frac{CT}{\sigma\eps}.
\]
For the term $\log(2T)/\sigma$, this uses $T\ge1/\sigma$, which gives $1/\sigma\le\sqrt{T/\sigma}$; for the remaining terms, it uses $\sigma\le1$ and $\log(CT/(\sigma\eps))\ge1$.
Moreover, $T/(\sigma\eps)\ge16$ gives $\log(CT/(\sigma\eps))\lesssim\log(2T/(\sigma\eps))$.
The learner does not depend on $\eps$, so taking the infimum over $\eps$ proves~\eqref{eq:main-rate-real-valued}.
\end{proof}

\section{\texorpdfstring{$L^\star$-dependent regret}{L-star-dependent regret}}\label{app:lstar}

In this section, we demonstrate that Gaussian FTPL is capable of achieving a regret bound that adapts to the loss of the best comparator in binary classification. In particular, we derive bound similar to~\cref{thm:main} in terms of the best comparator's number of mistakes, denoted by: 
\[
L^\star=\min_{f\in\cF}\sum_{t=1}^T\ind\{f(X_t)\ne Y_t\}.
\]
\begin{theorem}\label[theorem]{thm:lstar}
Consider a binary class $\cF \colon \cX \to \{\pm1\}$ of VC dimension $d\ge1$, and let $\ell(\hat y,y)=-\hat yy$ be the binary loss. For every horizon $T\ge1$ and every covariate process that is $\sigma$-smooth with respect to a fixed $\mu$, the learner in~\eqref{eq:ftpl} satisfies
\begin{equation}\label{eq:lstar-rate}
\En\reg_T \lesssim \sqrt{\frac T\sigma} \log^2(2T/\sigma)\left(\sqrt{d^2 \land \En L^\star} +\sqrt d\right).
\end{equation}
\end{theorem}
The learner does not require knowledge of $L^\star$.
In particular, in the realizable case, $L^\star = 0$ almost surely, thus, Gaussian FTPL recovers the bound \[
\En\reg_T\lesssim\sqrt{\frac{dT}\sigma}\log^2(2T/\sigma),
\]
which, up to logarithmic factors, is optimal in the realizable case~\citep{blanchard2025agnostic}. Note that, in order to establish~\cref{thm:lstar}, it suffices to prove: 
\[
\En\reg_T \lesssim \sqrt{\frac T\sigma} \log^2(2T/\sigma) (\sqrt{\En L^\star} +\sqrt d),
\]
since the second point in the minimum is given by~\cref{thm:main}. The proof of the above statement is shorter than~\cref{thm:main} and is more reminiscent of the proof for ERM~\citep{block2024performance}.
We first bound the expected number of mistakes that Gaussian FTPL makes on the past covariates.

\begin{lemma}\label[lemma]{lemma:lstar-loss}
Let $\cF\subseteq\{\pm1\}^{\cX}$ have VC dimension $d\ge1$.
Fix a history $(X_{1:t-1},Y_{1:t-1})$, and write $G=G^{(t)}$.
The forecaster in~\eqref{eq:ftpl} satisfies
\begin{equation}\label{eq:lstar-historical-loss}
\En_G\sum_{s<t}\ind\{\hat f_t(X_s)\ne Y_s\} \le 2\min_{f\in\cF}\sum_{s<t}\ind\{f(X_s)\ne Y_s\}+2d\log(2t).
\end{equation}
\end{lemma}

\begin{proof}
The claim is immediate for $t=1$.
Fix $t\ge2$ and $g\in\cF$.
All expectations are over $G$, with the history held fixed.
Since $(f(X_s)-Y_s)^2=2(1-Y_sf(X_s))$, optimality of $\hat f_t$ in~\eqref{eq:ftpl} gives the basic inequality
\begin{align*}
\sum_{s<t}(\hat f_t(X_s)-Y_s)^2 &\le \sum_{s<t}(g(X_s)-Y_s)^2 +2\sum_{s<t}G_s(g(X_s)-\hat f_t(X_s)).
\end{align*}
Split $g(X_s)-\hat f_t(X_s)=(g(X_s)-Y_s)+(Y_s-\hat f_t(X_s))$ and subtract half the learner's squared error.
Taking expectations yields: 
\begin{align*}
\frac12\En_G\sum_{s<t}(\hat f_t(X_s)-Y_s)^2 &\le \sum_{s<t}(g(X_s)-Y_s)^2\\&\quad+\En_G\max_{f\in\cF}\left\{2\sum_{s<t}G_s(Y_s-f(X_s))-\frac12\sum_{s<t}(f(X_s)-Y_s)^2\right\}.
\end{align*}
The last term can be interpreted as an offset Gaussian complexity~\citep{liang2015learning}. From Gaussian MGF, we have:
\[
\En_G\exp\left(\frac12\sum_{s<t}G_s(Y_s-f(X_s))-\frac18\sum_{s<t}(f(X_s)-Y_s)^2\right)=1.
\]
Applying Jensen's inequality conditional on the values $X_{1:t-1}$ and summing over the distinct prediction vectors on these covariates gives
\begin{align*}
&\En_G\max_{f\in\cF}\left\{2\sum_{s<t}G_s(Y_s-f(X_s))-\frac12\sum_{s<t}(f(X_s)-Y_s)^2\right\} \\
&\quad \le 4\log \En_G \exp\left(\max_{f\in\cF}\left\{\frac12\sum_{s<t}G_s(Y_s-f(X_s))- \frac18\sum_{s<t}(f(X_s)-Y_s)^2\right\}\right) \\
&\quad \le 4\log \En_G \sum_{v\in\left.\cF\right|_{X_{1:t-1}}} \exp\left\{\frac12\sum_{s<t}G_s(Y_s-v_s)-\frac18\sum_{s<t}(v_s-Y_s)^2\right\} \\
&\quad\le 4\log\left|\left.\cF\right|_{X_{1:t-1}}\right|
\le 4d\log(2t),
\end{align*}
where the last inequality follows from Sauer--Shelah (part~(i) of~\cref{prop:classical-covers}).
Substituting this bound and using $(f(X_s)-Y_s)^2=4\ind\{f(X_s)\ne Y_s\}$ gives
\[
\En_G\sum_{s<t}\ind\{\hat f_t(X_s)\ne Y_s\}
\le 2\sum_{s<t}\ind\{g(X_s)\ne Y_s\}+2d\log(2t).
\]
Minimizing over $g\in\cF$ proves the claim.
\end{proof}

The proof of the regret bound is concluded by applying the surprise lemma~\cref{lemma:surprise} to translate~\cref{lemma:lstar-loss} into a guarantee on the \emph{current} prediction loss.

\begin{proof}[Proof of~\cref{thm:lstar}]
The idea is to apply~\cref{lemma:surprise} to the process $(X_t,Y_t)$. First, note that conditional law of $(X_t,Y_t)$ given $\cH_{t-1}$ is $\sigma/2$-smooth with respect to $\mu\otimes\unif(\{\pm1\})$. Let
\[
h_t(x,y) := \En_{G^{(t)}} \frac{1- y \hat f_t(x)}{2} = \En_{G^{(t)}} \ind\{\hat f_t(x) \neq y\}.
\]
Note that $h_t$ is a predictable test function which belongs to the signed convex hull of the class
\[
\{(x,y)\mapsto yf(x):f\in\cF\} \cup\{1\},
\]
whose VC dimension is in $O(d)$.
Since $h_t^2\le h_t$, we have
\begin{equation}\label{eq:lstar-historical-moments}
\begin{aligned}
\En\sum_{t=1}^T\sum_{s<t}h_t(X_s,Y_s)^2 &\le\En\sum_{t=1}^T\En_{G^{(t)}}\sum_{s<t}\ind\{\hat f_t(X_s)\ne Y_s\}\\
&\lesssim T\En L^\star+dT\log(2T),
\end{aligned}
\end{equation}
where the second transition used~\cref{lemma:lstar-loss}. Therefore, using~\cref{lemma:surprise}, we have
\begin{align*}
\En \sum_{t=1}^T \ind \{\hat f_t(X_t) \neq Y_t\} &=
\En\sum_{t=1}^T h_t(X_t,Y_t) \\ &\lesssim\sqrt{\frac{T\log(2T)}\sigma\left(\En L^\star+d\log(2T)\right)}+\sqrt{\frac{dT}\sigma}\log^2(2T/\sigma)\\
&\lesssim\sqrt{\frac{T\log(2T)}\sigma\En L^\star}+\sqrt{\frac{dT}\sigma}\log^2(2T/\sigma).
\end{align*}
The first equality uses the conditional independence of $(X_t,Y_t)$ and $G^{(t)}$ given $\cH_{t-1}$.
Since $\reg_T$ is at most a factor of $2$ times number of mistakes, the inequality above concludes the proof.
\end{proof}

\end{document}